\documentclass[10pt,a4paper]{article}

\usepackage[margin=1in]{geometry}

\usepackage{amsmath,amssymb,amsfonts,amsthm,mathtools,latexsym}
\usepackage{rotating}
\usepackage{graphics}
\usepackage{graphicx}
\usepackage{booktabs}
\usepackage{multirow} 
\usepackage{url,bm,comment}

\usepackage{authblk}
\usepackage[titletoc,toc,title]{appendix}

\newtheorem{theorem}{\bf Theorem}

\newtheorem{lemma}{\bf Lemma}
\newtheorem{definition}{\bf Definition}

\newtheorem{remark}{\bf Remark}

\makeatother

\newcommand{\bea}{\begin{eqnarray}}
\newcommand{\eea}{\end{eqnarray}}

\renewcommand{\P}{\mathbb{P}}
\newcommand{\E}{\mathbb{E}}

\newcommand{\R}{\mathcal R}

\newcommand{\Pp}{\mathbb P}

\newcommand{\equaref}[1]{(\ref{eq:#1})}

\newcommand{\enne}
{
\hbox{I$\!$N}
}

\newcommand{\diff}{{\rm\,d}}

\newcommand{\Prob}
{
{\mathbb P}
}

\newcommand{\tgifeps}[3]{
\begin{figure}[htb]
\centering
\includegraphics[width=#1cm]{#2.eps}
\caption{#3\label{fig:#2}}
\vspace{-1mm}
\end{figure}
}

\title{On the Emergence of Shortest Paths by Reinforced \\
Random Walks}
\author{Daniel R. Figueiredo}
\affil{Computer Science and Syst. Eng. Dept.\\
Federal University of Rio de Janeiro (UFRJ), Brazil}
\author{Michele Garetto}
\affil{Computer Science Department\\
University of Torino, Italy}
\date{}                     
\begin{document}

\maketitle

\begin{abstract}
The co-evolution between network structure and functional performance is a fundamental and challenging 
problem whose complexity emerges from the intrinsic interdependent nature of structure and function.
Within this context, we investigate the interplay between the efficiency of network navigation (i.e., path 
lengths) and network structure (i.e., edge weights). We propose a simple and tractable model based on 
iterative biased random walks where edge weights increase over time as function of the traversed path length. 
Under mild assumptions, we prove that biased random walks will eventually only traverse shortest 
paths in their journey towards the destination. We further characterize the transient regime proving 
that the probability to traverse non-shortest paths decays according to a power-law. We also highlight various 
properties in this dynamic, such as the trade-off between exploration and convergence,  
and preservation of initial network plasticity. We believe the proposed model and results can be of interest to 
various domains where biased random walks and de-centralized navigation have been applied. 
\end{abstract}

\section{Introduction}\label{sec:introduction}

The interplay between network structure (nodes, edges, weights) and 
network function (high level features enabled by the network) is a fundamental and challenging 
problem present in a myriad of systems ranging from biology to economics and sociology.
In many complex systems network structure and network function co-evolve interdependently: 
while network structure constraints functional performance, the drive for functional efficiency
pressures the network structure to change over time. 
Within this tussle, {\it network activity} (i.e., basic background processes running
on the network) plays a key role in tying function and structure: 
in one hand, function execution often requires network activity, while in the other 
hand network structure often constraints network activity.

Given the complexity of co-evolution, simple and tractable models are often used to understand 
and reveal interesting phenomena. In this paper, we focus on {\em network navigation}, proposing 
and analyzing a simple model that captures the interplay between function and structure.
Our case-study embodies repetition, plasticity, randomization, valuation and memory which are 
key ingredients for evolution: repetition and memory allow for learning; plasticity and randomization
for exploring new possibilities; valuation for comparing alternatives. 
Moreover, in our case-study co-evolution is enabled by a single and simple network 
activity process: {\em biased random walks}, where time-varying edge weights play the role of memory. 

Network navigation (also known as routing) refers to the problem of finding short 
paths in networks and has been widely studied due to its importance in various contexts. Efficient 
network navigation can be achieved by running centralized or distributed algorithms. Alternatively, 
it can also be achieved when running simple greedy algorithms over carefully crafted 
network topologies. But can efficient navigation emerge without computational
resources and/or specifically tailored topologies? 

A key contribution of our work is to answer affirmatively the above question by means of 
Theorem~\ref{maintheo}, which states that under mild conditions 
efficient network navigation always emerges through the repetition of extremely 
simple network activity.   
More clearly, a biased random walk will eventually only take paths of minimum length, 
independently of network structure and initial weight assignment. Beyond 
its long term behavior, we also characterize the system transient regime, revealing interesting 
properties such as the power-law decay of longer paths, and the (practical) preservation 
of initial plasticity on edges far from ones on the shortest paths. 
The building block for establishing the theoretical results of this paper
is the theory of P{\'o}lya urns, applied here by considering a network of urns.

We believe the proposed model and its analysis could be of interest to various domains 
where some form of network navigation is present and where random walks are used as the 
underlying network activity, 
such as computer networking~\cite{passarella2012survey,huang2006data,leibnitz2006biologically}, 
animal movement in biology~\cite{codling2008random,smouse2010stochastic}, 
memory recovery in the brain~\cite{sadeh2015emergence,abbott2012human,saerens2009randomized}. 
Moreover, our results can enrich existing theories such as
Ant Colony Optimization (ACO) meta-heuristic~\cite{Dorigo:2004,dorigo2005ant}, 
Reinforcement Learning (RL) theory~\cite{sutton1998reinforcement}, and Edge Reinforced Random 
Walks (ERRW) theory~\cite{davis1990reinforced,pemantle2007survey} -- see related work in Section \ref{sec:related}.


\section{Model}
\label{sec:model}

We consider an arbitrary (fixed) network ${\cal G}=({\cal V},{\cal E})$, 
where $\cal V$ is a set of vertices and $\cal E$ is a set of directed edges among the vertices. 
We associate a weight (a positive real value) $w_{i,j}$ to every directed edge $(i,j) \in {\cal E}$.
Edge weights provide a convenient and flexible abstraction for structure, specially when 
considering evolution. Finally, a pair of fixed nodes $s$, $d \in  {\cal V}$ 
are chosen to be the source and destination, respectively.  
But how to go from $s$ to $d$? 

We adopt a very simple network activity model to carry out the function of 
navigation: weighted random walks (WRW). 
Specifically, a sequence of random walks, indexed by $n = 1,2,\ldots$, is executed 
on the network one after the other. 
Each WRW starts at $s$ and steps from node to node until it hits $d$. 
At each visited node, the WRW randomly follows an outgoing edge with 
probability proportional to its edge weight. 
We assume that weights on edges remain constant during 
the execution of a single WRW, 
and that decisions taken at different nodes are independent from each other. 

Once the WRW reaches the destination and stops, edge weights are updated, thus 
impacting the behavior of the next WRW in the sequence. 
In particular, edges on the path followed by the WRW are 
rewarded (reinforced) by increasing their weights with a positive amount 
which depends on the length of the path taken (expressed in number of hops).
Let $f:\enne^+ \rightarrow \R^+$ be some positive function of the path length, 
hereinafter called the reward function.   

We consider two different ways in which edges are reinforced: 
\begin{itemize}
   \item {\bf single-reward model}: each edge belonging to the path
   followed by the WRW is rewarded once, according to function $f(\cdot)$.
   \item {\bf multiple-reward model}: each edge belonging to the path
   followed by the WRW is rewarded according to function $f(\cdot)$  
   {\em for each time} the edge was traversed.
\end{itemize}   

Throughout the paper, we will interpret $n$, the number of random walks that 
have gone from $s$ to $d$ as a discrete time step. Thus, by co-evolution of 
the system we actually mean what happens to the network structure (i.e., weights) 
and navigation (i.e., path lengths) as $n \to \infty$.

Let $w_{i,j}[n]$ be the weight on edge $(i,j)$ at time $n$ (right after the 
execution of the $n$-th WRW but before the $(n+1)$-th WRW starts).
Let ${\cal P}_n$ denote the sequence of edges (i.e., the path) traversed by the $n$-th WRW and 
$L_n = |{\cal P}_n|$ the path length (in number of hops)\footnote{In this paper
we assume that the cost to traverse any edge is equal to 1, but results can be 
immediately generalized to the case in which a generic (positive) cost $c_{i,j}$ 
is associated to each edge $(i,j)$.}.

After reaching the destination, the weight of any distinct edge $(i,j)$ 
in ${\cal P}_n$ is updated according to the rule:
\begin{equation}\label{eq:update}
w_{i,j}[n] = w_{i,j}[n-1] + u_{i,j}({\cal P}_n) \cdot f(L_n) 
\end{equation}
where $u_{i,j}({\cal P}_n) = 1$ under the single-reward model, whereas $u_{i,j}({\cal P}_n)$ equals
the number of times edge $(i,j)$ appears in ${\cal P}_n$, under the multiple reward model.
We also allow for the event that the WRW does not reach the destination 
because it \lq gets lost' in a part of the network from which $d$ is no longer
reachable. In this case, we assume that no edge is updated by the WRW who fails to reach $d$.   

Note that our model has the desirable ingredients for co-evolution: edge set $\cal E$ 
and initial weights provide plasticity and WRW provides randomization, which allows 
for exploring alternative paths; edge weights provide memory and the sequence of 
WRW provides repetition, which enables learning; path length taken by WRW provides 
valuation, which allows for comparing alternatives paths. Moreover, note that 
functional performance induces structural changes through network activity as 
navigation (traversed path) changes edge weights, while network structure 
constraints function, as edge weights influence observed path lengths. 
Thus, our model captures the essence of co-evolution. But will efficient 
navigation emerge? In particular, which paths ${\cal P}_n$ are taken 
as $n$ increases? 


\section{Related work}
\label{sec:related}

The problem of finding shortest paths in networks is, of course, a well
understood problem in graph theory and computer science, for which
efficient algorithms are available, both centralized (e.g., Dijkstra)
and distributed (e.g., Bellman-Ford). Our approach follows in the second
category (distributed), as it does not require knowledge of the topology,
however we stress that our goal is not to propose yet another
way to compute shortest paths in network (actually, the convergence 
of our process is slower than that of Bellman-Ford), but to show that
shortest paths can naturally emerge from the repetition of a
simple and oblivious network activity which does not require
computational/memory resources on the nodes. As such, our model
is more tailored to biological systems, rather than technological networks.

The celebrated work of Kleinberg~\cite{klein} was probably the first 
to show that efficient navigation is indeed feasible 
through a simple greedy strategy based solely on local information,
but under the stringent assumption that the network exhibits a very 
particular structure. Greedy algorithms can also lead to 
efficient network navigation under distributed hash tables (DHTs), 
but again this requires the network to exhibit a very particular 
topology~\cite{stoica2003chord}. 

The idea of reinforcing edges along paths followed by random walks 
is surely reminiscent of Ant Colony Optimization (ACO), a 
biologically-inspired meta-heuristic for exploring the solution space of complex
optimization problems which can be reduced to finding good paths through graphs
~\cite{Dorigo:2004,dorigo2005ant}. Although some versions of ACO can be proved to 
converge to the global optimum, their analysis turns out to be complicated and 
mathematically non-rigorous, especially due to {\em pheromone evaporation} (i.e., weights 
on edges decrease in the absence of reinforcement). Moreover, like most meta-heuristics,
it is very difficult to estimate the theoretical speed of convergence. 
In contrast to ACO, our model is simpler and has the modest goal of revealing 
shortest paths in a given network, instead of exploring a solution space.
Moreover, we do not introduce any evaporation, and we exploit totally
different techniques (the theory of P{\'o}lya urns) to establish our results,
including the transient behavior (convergence) of the system.

In Reinforcement Learning (RL), the problem of finding an optimal policy through a random 
environment has also been tackled using Monte Carlo methods that reinforce actions based 
on earned rewards, such as the $\epsilon$-soft policy algorithm~\cite{sutton1998reinforcement}. 
Under a problem formulation with no terminal states and expected discounted rewards, it can be 
rigorously shown that an iterative algorithm converges to the optimal policy~\cite{tsitsiklis2003convergence}. 
However, in general and more applicable scenarios, the problem of convergence to optimal policies 
is still an open question, with most algorithms settling for an approximate solution. 
Although lacking the notion of action set, 
our model is related to RL in the sense that it aims at finding
paths accumulating the minimum cost, through an unknown environment, 
using a Monte-Carlo method. Our convergence results (convergence in probability) and the 
techniques used in the analysis (P{\'o}lya urns) are fundamentally different from what is commonly
found in RL theory, and could thus be useful to tackle problems in this area. 



Edge Reinforcement Random Walks (ERRW) is a mathematical modeling framework consisting 
of a weighted graph where weights evolve over time according to steps taken by a random 
walker~\cite{davis1990reinforced,pemantle2007survey}. 
In ERRW, a single random walk moves around without having any destination 
and without being restarted. Moreover, an edge weight is updated immediately after traversal of the edge,
according to functions based on local information. 
Mathematicians have studied theoretical aspects of ERRW such as the convergence of the network 
structure (relative weights) and the recurrence behavior of the walker (whether it will continue to visit 
every node in the long run, or get trapped in one part). Similarly to our model, a key ingredient 
in the analysis of ERRW is the P{\'o}lya urn model, specially on directed networks. 
In contrast to us, ERRW model was not designed to perform any particular function and thus does not 
have an objective. Our model is substantially different from traditional ERRW, and we believe it could 
suggest a concrete application as well as new directions to theoreticians working on ERRW. 

Animal movement is a widely studied topic in biology to which probabilistic models have 
been applied, including random walk based models~\cite{codling2008random,smouse2010stochastic}. 
In particular, in the context of food foraging, variations of ERRW models have been used to capture 
how animals search and traverse paths to food sources. A key difference in such variations is a 
{\em direction vector}, an information external to the network (but available on all nodes) that 
provides hints to the random walk. Such models have been used to show the emergence of relatively 
short paths to food sources, as empirically observed with real (monitored) animals. 
In contrast, we show that shortest paths (and not just short) can emerge even when external  
information is not available.


Understanding how neurons in the brain connect and fire to yield higher level functions 
like memory and speech is a fundamental problem that has recently received much 
attention and funding~\cite{seung2012connectome,hbp}. 
Within this context, random walk based models have been proposed and 
applied~\cite{saerens2009randomized,abbott2012human} along with models 
where repeated network activity modifies the network structure~\cite{sadeh2015emergence}. 
In particular, the latter work considers a time varying weighted network model under a more 
complex rule (than random walks) for firing neurons to show that the network structure 
can arrange itself to perform better function. We believe our work can provide building blocks 
in this direction since our simple model for a time varying (weighted) network also 
self-organizes to find optimal paths. 

Biased random walks have also been applied to a variety of computer networking 
architectures~\cite{passarella2012survey,huang2006data,leibnitz2006biologically}, with the goal 
of designing self-organizing systems to locate, store, replicate and manage
data in time-varying scenarios. We believe our model and findings could be of interest in 
this area as well. 


\section{Main finding}\label{sec:finding}
Let $L_{\min}$ be the length of the shortest path in graph $\cal G$ 
connecting source node $s$ to destination node $d$.
Denote by ${\cal P}_n$ the path taken by the $n$-th WRW, and 
by $\cal P$ an arbitrary path from $s$ to $d$, of length 
$L_{\cal P}$.

\begin{theorem}\label{maintheo}
Given a weighted directed graph $\cal G$,
a fixed source-destination pair $s$-$d$ (such that $d$ is reachable from $s$),
an initial weight assignment (such that all initial weights are positive),
consider an arbitrary path $\cal P$ from $s$ to $d$.
Under both the single-reward model and the multiple-reward model,
provided that the reward function $f(\cdot)$ is a strictly decreasing
function of the path length, as the number $n$ of random walks performed on the graph
tends to infinity, we have:
\begin{eqnarray*}
\lim_{n \rightarrow \infty} \Prob\{{\cal P}_n = {\cal P}\} = 
\begin{cases} c({\cal P}), & \mbox{if } L_{\cal P} = L_{\min} \\
0 , & \mbox{if } L_{\cal P} > L_{\min} 
\end{cases}
\end{eqnarray*}
where $c({\cal P})$ is a random variable taking values in $(0,1]$, 
that depends on the specific shortest path $\cal P$.
\end{theorem}

The above theorem essentially says that {\em all} shortest paths
are taken with non-vanishing probability, while {\em all} 
non-shortest paths are taken with vanishing probability, 
as $n \rightarrow \infty$. Note however that the probability 
that a specific shortest path is taken is a random variable, 
in the sense that it depends on the \lq system run' (system sample path). 

\begin{remark}[]
The asymptotic property stated in Theorem \ref{maintheo} is very robust,
as it holds for any directed graph, any strictly decreasing function $f(\cdot)$, 
and any (valid) initial weights on the edges. Note instead that the (asymptotic) 
distribution of $c({\cal P})$, for a given shortest path $\cal P$,
as well as the convergence rate to it, depends strongly on the update function $f(\cdot)$,
on the graph structure, and on the initial conditions on the edges.
\end{remark} 

\begin{remark}[]
We will see in the proof of of Theorem \ref{maintheo} that the assumption 
of having a strictly decreasing function $f(\cdot)$ can be partially relaxed, 
allowing the reward function to be non-increasing for
$L > L_{\min}$. 
\end{remark} 

\section{Preliminaries}\label{sec:prel}

\subsection{Definitions}\label{subsec:defi}
The following definitions for nodes and edges play a central role in our analysis.
\begin{definition}[\bf decision point]
A decision point is a node $i \in {\cal V}$, reachable by $s$, that has more than 
one outgoing edge that can reach $d$.
\end{definition}
\begin{remark}[]
Clearly, we can restrict our attention to nodes that are decision points, 
since all other nodes are either never reached by random walks originating in $s$, 
have zero or one outgoing edge (having no influence on the random walk behavior), 
or their outgoing edges are never reinforced since the destination cannot be reached from them. 
\end{remark} 
\begin{definition}[{\bf $\alpha$-edge and $\beta$-edge}] 
An outgoing edge of decision point $i$ is called an $\alpha$-edge 
if it belongs to some shortest path from $i$ to $d$, whereas it is called a $\beta$-edge 
if it does not belong to any shortest path from $i$ to $d$.
\end{definition}
Note that every outgoing edge of a decision point is either an $\alpha$-edge or $\beta$-edge. 
Let $q_\alpha(i,j,n)$ denote the probability that the random walk, at time $n$, takes a shortest path 
from $i$ to $d$ after traversing the $\alpha$-edge $(i,j)$. 
Let $q_\beta(i,j,n)$ denote the probability that 
the random walk, at time $n$, will {\em not} return back to node $i$
after traversing the $\beta$-edge $(i,j)$.
Note that the above probabilities depend, in general, on the considered edge, 
on the network structure and on the set of weights at time $n$.  

\begin{definition}[{\bf $\alpha^*$-edge and $\beta^*$-edge}]
An $\alpha^*$-edge is an $\alpha$-edge such that, after traversing it, the random walk 
takes a shortest path to $d$ with probability 1, and thus $q_\alpha(i,j,n)=1$. 
A $\beta^*$-edge is a $\beta$-edge such that, after traversing it, the random walk 
does not return to node $i$ with probability 1, and thus $q_\beta(i,j,n) = 1$. 
\end{definition}
Note that $\alpha^*$-edge and $\beta^*$-edge can occur due solely to topological constraints.
In particular, we have an $\alpha^*$-edge whenever the random walk, after traversing
the edge, can reach $d$ only through paths of minimum length.   
In a cycle-free network, all $\beta$-edges are necessarily $\beta^*$-edges.

\subsection{The single decision point}\label{subsec:onepoint}
As a necessary first step, we will consider the simple case in which there is 
a single decision point in the network. The thorough analysis of this scenario 
provides a basic building block towards the analysis of the general case. 

We start considering the simplest case in which there are 
two outgoing edges (edge 1 and edge 2) from the decision point, 
whose initial weights are denoted by $w_1[0]$ and $w_2[0]$, respectively.
Let $L_1$ and $L_2$ denote the (deterministic) length of the path experienced by random walks 
when traversing edge 1 and edge 2, respectively. Correspondingly, let 
$\Delta_1 = f(L_1)$ and $\Delta_2 = f(L_2)$ denote the rewards given the 
edge 1 and edge 2, respectively.

The mathematical tool used here to analyze this system, especially its 
asymptotic properties, are P\'{o}lya urns~\cite{book}. 
The theory of P\'{o}lya urns is concerned with the evolution of the 
number of balls of different colors (let $K$ be the number of colors) 
contained in an urn from which we repeatedly draw one ball uniformly at random.
If the color of the ball withdrawn is $i$, $i = 1,\ldots,K$, then $A_{i,j}$   
balls of color $j$ are added to the urn, $j = 1,\ldots,K$, in addition to the
ball withdrawn, which is returned to the urn. In general, $A_{i,j}$ can be
deterministic or random, positive or negative. Let ${\bm A }$ be the
matrix with entries $A_{i,j}$, usually referred to as the {\em schema}
of the P\'{o}lya urn.


We observe that a decision point can be described by a P\'{o}lya urn, 
where the outgoing edges represent colors, the edge weight is the number 
of balls in the urn\footnote{Although P\'{o}lya urn models have been 
traditionally developed considering an integer number of balls for each 
color, analogous results hold in the case of real numbers, 
when all $A_{i,j}$ are positive (as in our case).}, 
and entries $A_{i,j}$ correspond to edge reinforcements 
according to taken path lengths (through function $f(\cdot)$).
In the simple case with only two edges, we obtain the following schema:
\begin{equation}
{\bm A }  =  \left( \begin{array}{cc}
\Delta_1 & 0 \\
0 & \Delta_2 \\
\end{array} \right) 
\end{equation}
We first consider the situation in which $\Delta_1 = \Delta_2$, which occurs
when both edges are part of a shortest path, and thus, both edges are $\alpha^*$-edges. 
A classical result in P\'{o}lya urns states that the normalized
weight of edge 1 (similarly for edge 2), i.e., the weight on edge 1 divided 
by the sum of the weights, tends in distribution to a beta distribution:
\begin{equation}\label{eq:beta1}
\frac{w_1[n]}{w_1[n]+w_2[n]} \xrightarrow[]{\cal D} \beta\left(\frac{w_1[0]}{\Delta_1},\frac{w_2[0]}{\Delta_2}\right) 
\end{equation}
Note that in this simple case the above beta distribution
completely characterizes the asymptotic probability
of traversing the shortest path comprising edge 1 (or edge 2).
Hence, we obtain a special case of the general result stated in 
\ref{maintheo}, where the random variable $c({\cal P})$ is a 
beta distribution which depends both on the update function and the initial 
weights. Informally, we say that both shortest paths will always \lq survive', 
as they will be asymptotically used with a (random) non-zero probability, 
independent of the sample path taken by the system. 

The above result can be directly generalized to the case of $K$ outgoing
edges, all belonging to shortest paths. Indeed, let's denote the
asymptotic normalized weight of edge $i$ by $r_i$:
$$r_i = \lim_{n \rightarrow \infty} \frac{w_i[n]}{\sum_{j=1}^{K} w_j[n]}$$
Moreover, let $\alpha_i = \frac{w_i[0]}{\Delta_i}$.
Then it is known that the joint probability density function of the $r_i$'s 
tends to a Dirichlet distribution with parameters $\{\alpha_i\}$. 

A useful property of the Dirichlet distribution is aggregation: if
we replace any two edges with initial weights $w_i$, $w_j$ by a single edge with 
initial weight $w_1 + w_2$, we obtain another Dirichlet distribution where 
the \lq combined' edge is associated to parameter $\alpha_i + \alpha_j$, i.e., if
${\bm r} = (r_1,\ldots,r_K) \sim {\rm Dirichlet}(\alpha_1,\ldots,\alpha_K)$
then ${\bm r'} = (r_1,\ldots,r_i + r_j,\ldots r_K) \sim {\rm Dirichlet}(\alpha_1,\ldots,\alpha_i+\alpha_j,\ldots\alpha_K)$.
Note that the marginal distribution with respect to any of the edges
is, as expected, a beta distribution, i.e.,
$$ r_i \sim \beta \left(\alpha_i,\sum_{j=1,j\neq i}^K \alpha_i\right) $$

Let's now consider outgoing edges that lead to paths of different
lengths, starting from the simple situation in which we have just two  
edges. Without lack of generality, let's assume that $\Delta_1 > \Delta_2$ 
in which case edge 1 is an $\alpha^*$-edge and edge 2 is a $\beta^*$-edge.
The analysis of the corresponding P\'{o}lya urn model uses a technique 
known as {\em Poissonization}~\cite{book}. The basic idea is to embed
the discrete-time evolution of the urn in continuous time, associating to each
ball in the urn an independent exponential timer with parameter 1. When a 
timer \lq fires', the associated ball is drawn, and we immediately perform the 
corresponding ball additions (starting a new timer for each added ball). 
The memoryless property of the exponential distribution guarantees that 
the time at which a ball is drawn is a renewal instant for the system. Moreover, 
competition among the timers running in parallel exactly produces the desired 
probability to extract a ball of a given color at the next renewal instant. This 
means that, if $t_n$ is the (continuous) time at which the $n$-th timer fires,
at time $t_n$ the number of balls in the continuous-time system has {\em exactly} 
the same distribution as the number of balls in the original discrete-time
system after $n$ draws. It follows that the asymptotic behavior (as $t \rightarrow \infty$)
of the continuous-time system coincides with the asymptotic behavior
of the discrete-time system (as $n \rightarrow \infty$), but the 
continuous-time system is more amenable to analysis, thanks to 
the independence of all Poisson processes in the urn.

The Poissonization technique leads to the following fundamental result: 
Let ${\bm w}(t)$ be the (column) vector of edge weights
at time $t$ in the continuous-time system. We have (Theorem 4.1 in \cite{book}):
${\bm E }[{\bm w}(t)] =  e^{{\bm A}^T t} \, {\bm w}(0)$.
The above result can be extended to the case in which the entries
of schema ${\bm A}$ are independent random variables
(independent among them and from one draw to another) by simply 
substituting ${\bm A}^T$ with ${\bm E}[{\bm A}^T]$: 
\begin{equation}\label{eq:poissonization}
{\bm E }[{\bm w}(t)] =  e^{{\bm E}[{\bm A}^T] t} \, {\bm w}(0)
\end{equation} 
i.e., by considering a schema in which random entries
are replaced by their expectations. This extension will be particularly 
useful in our context.

\section{Asymptotic analysis}\label{sec:asym}
In this section we prove Theorem \ref{maintheo} first constrained to 
directed acyclic graphs (DAG), then relaxing to general 
topologies under the multiple-reward model and finally to 
the single-reward model. 

\vspace{-0mm}
\subsection{The DAG case}\label{subsec:dag}
Let ${\cal G}$ be a directed acyclic graph (DAG) and note that 
in this case edges are either $\alpha$-edges or $\beta^*$-edges. 
Moreover, the absence of cycles forbids traversing an edge more than once, 
so the single-reward model coincides with the multiple-reward model. 

We first introduce the following key lemma.
\begin{lemma}\label{lem1}
Consider a decision point having one or more \mbox{$\alpha^*$-edges} 
and one or more $\beta^*$-edges. The normalized weight of 
any $\beta^*$-edge vanishes to zero as $n \rightarrow \infty$.
\end{lemma}
\begin{proof}
Let $\hat{L}$ denote the length of the shortest path from the decision point to 
$d$. Note that this path length is realized by the random walk after following 
an $\alpha^*$-edge. Observe that $\alpha^*$-edges can be merged together into a single 
virtual \mbox{$\alpha^*$-edge} whose weight, denoted by $\hat{w}$, 
is defined as the sum of the weights of the merged $\alpha^*$-edges. 
Similarly, we will merge all $\beta^*$-edges into a single virtual 
$\beta^*$-edge of weight $\dot{w}$, defined as the sum of the 
weights of the merged $\beta^*$-edges.

Let $\{Z_n, n\geq 1\}$ be the stochastic process corresponding to 
$Z_n = \frac{\dot{w}[n]}{\dot{w}[n] + \hat{w}[n]}$,
i.e., $Z_n$ is the normalized weight of the virtual merged $\beta^*$-edge
after the $n$-th random walk. We are going to show that $\lim_{n \to \infty} Z_n = 0$ 
which implies that the asymptotic probability to follow any $\beta^*$-edge goes 
to zero as well. 
The proof is divided into two parts. First, we show that $\lim_{n \to \infty} Z_n$ 
exists almost surely, namely, $Z_n$ converges to a given constant $z \in [0,1]$. 
Second, we will show that $z$ can only be equal to 0.
For the first part, we will use Doob's Martingale Convergence Theorem \cite{durrett}, after proving that
$Z_n$ is a super-martingale. Since $\{Z_n\}$ is discrete time, and $0 \leq Z_n \leq 1$, it suffices
to prove that $\E[Z_{n+1}|{\cal F}_n] \leq Z_n$, where the filtration ${\cal F}_n$ corresponds to all available information after the $n$-th walk.
Now, the normalized weight, at time $n+1$, of any $\beta^*$-edge 
is stochastically dominated by the normalized weight, at time $n+1$, of the same $\beta^*$-edge 
assuming that it belongs to a path of length $\hat{L} + 1$. This is essentially the reason 
why we can merge all $\beta^*$-edges into a single virtual $\beta^*$-edge belonging 
to a path of length $\hat{L} + 1$.
Hence, $\E[Z_{n+1}|{\cal F}_n] \leq \E[Z'_{n+1}|{\cal F}_n]$, where $Z'_{n+1}$ is the
aggregate normalized weight of the virtual $\beta^*$-edge.
We proceed by considering what can happen when running the \mbox{$(n+1)$-th} walk. Two cases are possible:
i) either the random walk does not reach the decision point, in which case $Z'_{n+1} = Z_n$ since 
edge weights are not updated, or ii) it reaches the decision point having accumulated 
a (random) hop count $\ell_{n+1}$. 
In the second case, we can further condition on the value taken by $\ell_{n+1}$ and prove that
$\E[Z'_{n+1}|{\cal F}_n, \ell_{n+1}] \leq Z_n$, $\forall\, \ell_{n+1}$: \vspace{-2mm}
\begin{multline}
\E[Z'_{n+1}|{\cal F}_n, \ell_{n+1}] = 
Z_n \frac{\dot{w}(n) + f(\ell_{n+1} + \hat{L} + 1)} {\dot{w}(n)+\hat{w}(n)+f(\ell_{n+1} + \hat{L} + 1)} + 
(1-Z_n) \frac{\dot{w}(n)} {\dot{w}(n)+\hat{w}(n)+f(\ell_{n+1} + \hat{L})} = \\
Z_n \left[ \frac{\dot{w}(n) + f(\ell_{n+1} + \hat{L} + 1)} {\dot{w}(n)+\hat{w}(n)+f(\ell_{n+1} + \hat{L} + 1)} \right. 
\left. + \frac{\hat{w}(n)}{\dot{w}(n)} \frac{\dot{w}(n)} {\dot{w}(n)+\hat{w}(n)+f(\ell_{n+1} + \hat{L})} \right] \leq \\
Z_n \left[\frac{\dot{w}(n)+f(\ell_{n+1}+\hat{L} + 1)+ \hat{w}(n)}
{\dot{w}(n)+f(\ell_{n+1}+\hat{L} + 1)+ \hat{w}(n)}\right] = Z_n
\end{multline}
where the inequality holds because $f(\cdot)$ is assumed to be non-increasing.
 
At last, unconditioning with respect to $\ell_{n+1}$, whose distribution descends from ${\cal F}_n$, 
and considering also the case in which the random walk does not reach the decision point, 
we obtain $\E[Z'_{n+1}|{\cal F}_n] \leq Z_n$ and thus $\E[Z_{n+1}|{\cal F}_n] \leq Z_n$.
So far we have proven that $Z_n$ converges to a constant $z \in [0,1]$.
To show that necessarily $z = 0$, we employ the Poissonization technique
recalled in Section \ref{subsec:onepoint}, noticing again that $Z_n$ is stochastically
dominated by $Z'_n$. For the process $Z'_n$, we have:
$$ {\bm A}^T = \left( \begin{array}{cc} f(\ell + \hat{L}) & 0 \\
0 & f(\ell + \hat{L} + 1) \end{array} \right)$$
where $\ell$ is the (random) hop count accumulated at the decision point.
We will show later that the normalized weight of any edge in the network
converges asymptotically almost surely.
Hence, $\ell$ has a limit distribution, that we can use to compute expected 
values of the entries in the above matrix:
$$ \E[{\bm A}^T] = \left( \begin{array}{cc} \E_\ell[f(\ell + \hat{L})] & 0 \\
0 & \E_\ell[f(\ell + \hat{L} + 1)] \end{array} \right) = 
\left( \begin{array}{cc} a & 0 \\
0 & d \end{array} \right) $$
obtaining that $a > d$ when $f(\cdot)$ is strictly decreasing.
At this point, we can just apply known results of P\'{o}lya urns' 
asymptotic behavior (see Theorem 3.21 in~\cite{Janson1}),
and conclude that the normalized weight of the $\beta^*$-edge
must converge to zero.  
Alternatively, we can apply \equaref{poissonization} and observe that in this simple case 
\begin{equation}\label{eq:asymlem1}
\left( \begin{array}{c} 
{\bm E }[\hat{w}(t)] \\
{\bm E }[\dot{w}(t)] \end{array} \right) = 
\left( \begin{array}{cc}
e^{a t} & 0 \\
0 & e^{d t} \end{array} \right)
\left( \begin{array}{c} \hat{w}(0) \\ \dot{w}(0) \end{array} \right) 
\end{equation}
Therefore the (average) weight of the $\alpha^*$-edge
increases exponentially faster than the (average) weight of 
the $\beta^*$-edge.  
\end{proof}


Lemma \ref{lem1} provides the basic building block to prove Theorem \ref{maintheo}. 
\begin{proof}[Proof of Theorem \ref{maintheo} (DAG case)] 
We sequentially consider the decision points of the network
according to the partial topological ordering given by the 
hop-count distance from the destination. Simply put, we start 
considering decision points at distance 1 from the destination, then 
those at distance 2, and so on, until we hit the source node $s$. 
We observe that Lemma \ref{lem1} can be immediately applied to 
decision points at distance 1 from the destination.
Indeed, these decision points have one (or more) $\alpha^*$-edge, with $\hat{L}=1$,
connecting them directly to $d$, and zero or more $\beta^*$-edges connecting them to 
nodes different from $d$. 
Then, Lemma \ref{lem1} allows us to conclude that, asymptotically, the normalized
weight of the virtual $\alpha^*$-edge will converge to 1, whereas the normalized weight 
of all $\beta^*$-edges will converge to zero. This fact essentially allows us to
{\em prune} the $\beta^*$-edges of decision nodes at distance 1, 
and re-apply Lemma \ref{lem1} to decision points at distance 2 (and so on). Note that after the 
pruning, an \mbox{$\alpha$-edge} of a decision point at distance 2 necessarily becomes an 
\mbox{$\alpha^*$-edge}. As a consequence of the progressive pruning
of $\beta$-edges, we remove from the graph all edges which do not belong to 
shortest paths from a given node $i$ to $d$ (when we prune a $\beta$-edge, 
we contextually remove also edges that can only be traversed by following the 
pruned edge, and notice that by so doing we can also remove some $\alpha$-edge).

When the above iterative procedure hits the source node $s$, we are guaranteed
that only shortest paths from $s$ to $d$ remain in the residual graph (and all of them).
As a consequence, over the residual graph, a random walk starting from $s$ can only 
reach $d$ through a shortest path.  
Note that the normalized weight of any edge $(i,j)$ belonging to 
a shortest path will converge to a random variable $z_{i,j}$ bounded away from zero. 
Hence the asymptotic probability to follow any given shortest path ${\cal P}$, 
given by the product of normalized weights of its edges, will converge 
as well to a a random variable $c({\cal P})$ bounded away from zero.
Conversely, any path which is not a shortest path cannot \lq survive'. 
Indeed, any such path must traverse at least one decision point and 
take at least one $\beta^*$-edge. However, the above iterative procedure will 
eventually prune all \mbox{$\beta^*$-edges} belonging to the considered non-shortest path,
which therefore cannot survive.
\end{proof}

\vspace{-0mm}
\subsection{The multiple reward model in general network}\label{subsec:gen}
We now consider the case of an arbitrary directed graph possibly 
with nodes exhibiting (even multiple) self-loops. Moreover, we first focus on the 
multiple-reward model which is more challenging to analyze, and discuss 
the single-reward model in Section \ref{sec:single}.

Essentially, we follow the same reasoning as in the DAG case, by first 
proving a generalized version of Lemma \ref{lem1}.
\begin{lemma}\label{lem2}
Consider a decision point having one or more \mbox{$\alpha^*$-edges} 
and one or more $\beta$-edges. The normalized weight of 
any $\beta$-edge vanishes to zero as $n \rightarrow \infty$.
\end{lemma}
\begin{proof}
Similarly to the proof of Lemma \ref{lem1}, we merge all \mbox{$\alpha^*$-edges}
into a single virtual $\alpha^*$-edge with total weight $\hat{w}$. 
Moreover, we merge all $\beta$-edges into a single virtual 
$\beta$-edge with weight $\dot{w}$, defined as the sum of the 
weights of the merged $\beta$-edges. Such virtual $\beta$-edge 
can be interpreted as the best {\em adversary} against the virtual $\alpha^*$-edge.
Clearly, the best $\beta$-edge is an outgoing edge that (possibly) brings the random 
walk back to the decision point over the shortest possible cycle, i.e., a self-loop. 
It is instead difficult, {\em a priori}, to establish which 
is the best possible value of its parameter $q_\beta(n)$, i.e.,
the probability (in general dependent on $n$) the makes
the virtual $\beta$-edge the best competitor of the virtual $\alpha^*$-edge.
Therefore, we consider arbitrary values of $q_\beta(n) \in [0,1]$ (technically, if $q_\beta(n)>0$ 
then the $\beta$-edge cannot be a self-loop, but we optimistically assume that 
loops have length 1 even in this case).
In the following, to ease the notation, let $q = q_\beta(n)$. 
Similarly to the DAG case, we optimistically assume that if the random walk
reaches the destination without passing through the \mbox{$\alpha^*$-edge}, 
the overall hop count will be $\ell_{n+1} + i + \hat{L} + 1$, where $\ell_{n+1}$ 
is the hop count accumulated when first entering the decision point, while 
$i \geq 0$ denotes the number of (self) loops.
Instead, if the random walk reaches the destination 
by eventually following the $\alpha^*$-edge, the overall hop count will be $\ell_{n+1} + i + \hat{L}$.
In any real situation, the normalized cumulative weight $Z_n$ of $\beta$-edges 
is stochastically dominated by the weight of the virtual best adversary, having
normalized weight $Z'_n$.
We have:
\begin{multline}\label{eq:lem2}
\E[Z'_{n+1}|{\cal F}_n, \ell_{n+1}] = 
Z_n \left[ \sum_{i=0}^{\infty} [(1-q)Z_n]^i \left( \frac{\hat{w}(n)}{\dot{w}(n)} 
\frac{\dot{w}(n)+i \Delta(i)}{\dot{w}(n)+i \Delta(i)+\hat{w}(n) + \Delta(i)} + \right.\right. \\
\left.\left. q \frac{\dot{w}(n)+i \Delta'(i)}{\dot{w}(n)+i \Delta'(i)+\hat{w}(n)} \right) \right]
\end{multline}
where $\Delta(i) = f(\ell_{n+1} + i + \hat{L})$ and $\Delta'(i) = f(\ell_{n+1} + i + \hat{L} + 1)$.

Now, it turns out that the term in square brackets of the latter expression 
is smaller than or equal to one for any value of $\dot{w}(n)$, $\hat{w}(n)$, $\hat{L}$, $\ell_{n+1}$, $q$ 
and non-increasing function $f(\cdot)$. This property can be easily checked numerically, but a
formal proof requires some effort (see App. \ref{app:compl}).  
As a consequence, $\E[Z'_{n+1}|{\cal F}_n, \ell_{n+1}] \leq Z_n$. 
At last, unconditioning with respect to $\ell_{n+1}$, whose distribution descends from ${\cal F}_n$, 
and considering also the case in which the $(n+1)$-th random walk does not reach the decision point, 
we obtain $\E[Z'_{n+1}|{\cal F}_n] \leq Z_n$ and thus $\E[Z_{n+1}|{\cal F}_n] \leq Z_n$.
Hence, we have that $Z_n$ converges to a constant $z \in [0,1]$.

To show that necessarily $z = 0$, we employ the Poissonization technique as in 
Section \ref{subsec:onepoint}, noticing again that $Z_n$ is stochastically
dominated by $Z'_n$. For the process $Z'_n$, we have:
\begin{equation}\label{eq:abd}
{\bm E}[{\bm A}^T] = \left( \begin{array}{cc}
a & b \\
0 & d \end{array} \right)
\end{equation}
The entries in the above matrix have the following meaning:
\begin{itemize}
\item $a$ is the average reward given to the $\alpha^*$-edge if we
select the $\alpha^*$-edge;
\item $b$ is the average reward given to the $\alpha^*$-edge if we
select the $\beta$-edge;
\item $d$ is the average reward given to the $\beta$-edge if we
select the $\beta$-edge;
\end{itemize}
Note that the average reward given to the $\beta$-edge if we 
select the $\alpha^*$-edge is zero.

Luckily, the exponential of a 2x2 matrix in triangular form
is well known \cite{wso} (see also \cite{Janson2}
for limit theorems of triangular P\'{o}lya urn schemes).
In particular, when $a \neq d$ we obtain:
\begin{equation}\label{eq:asymlem2}
\left( \begin{array}{c} 
{\bm E }[\hat{w}(t)] \\
{\bm E }[\dot{w}(t)] \end{array} \right) = 
\left( \begin{array}{cc}
e^{a t} & \frac{b}{d-a} (e^{d t} - e^{a t}) \\
0 & e^{d t} \end{array} \right)
\left( \begin{array}{c} \hat{w}(0) \\ \dot{w}(0) \end{array} \right) 
\end{equation}
The special case in which $a = d$ will be considered later
(see Section \ref{sec:single}). 

To show that necessarily $z=0$, we reason by contradiction,
assuming that $Z'(n)$ converges to $z > 0$.
This implies that 
\begin{equation}\label{eq:sim}
\dot{w}(t) = \frac{z}{1-z} \hat{w}(t) + o(\hat{w}(t))
\end{equation} 
Moreover, we will assume that a large enough number of walks
has already been performed such that, for all successive walks,
the probability to follow the $\beta$-edge is essentially equal to $z$.
Specifically, let $n^*$ be a large enough time step such that 
the normalized weight of the $\beta$-edge is  $z - \epsilon < Z'(n) < z + \epsilon$
for all $n > n^*$.
We can then \lq restart' the system from time $n^*$, 
considering as initial weights $\dot{w}(n^*)$ and $\hat{w}(n^*)$
(the specific values are not important). 

Taking expectation of \equaref{sim} and plugging in the expressions
of the average weights in \equaref{asymlem2}, we have that the following 
asymptotic\footnote{Given two functions $f(n)$ and $g(n)$, we 
write $f(n)\sim_e g(n)$ if $\lim_{n\to\infty}\frac{f(n)}{g(n)}=1$.}
relation must hold:
$$ e^{d t} \dot{w}(n^*) \sim_e \frac{z}{1-z} \left( e^{a t} \hat{w}(n^*) 
 + \frac{b}{d-a} (e^{d t} - e^{a t})\dot{w}(n^*) \right) $$
Clearly, the above relation does not hold if $d < a$.
If $d > a$, the relation is satisfied when
\begin{equation}\label{eq:crucial}
\frac{b}{d-a} = \frac{1-z}{z} \Leftrightarrow \frac{d-a}{d-a+b} = z 
\end{equation}

Interestingly, we will see that \equaref{crucial} is verified when the reward function
is constant, suggesting that in this case the $\beta$-edge can indeed
\lq survive' the competition with the \mbox{$\alpha^*$-edge}.
Instead, we will show that \mbox{$\frac{d-a}{d-a+b} < z-\epsilon$}, 
for any strictly decreasing function $f(\cdot)$, proving that the normalized weight of the 
$\beta$-edge cannot converge to any $z > 0$.

For simplicity, we will consider first the the case in which
$\ell_{n+1}$, the hop count accumulated by the random walk while 
first entering the decision point,
is not random but deterministic and equal to $\ell$.
Under the above simplification, we have:
\begin{eqnarray}
& \hspace{-1cm} a = \hspace{-10mm} & f(\ell + \hat{L}) \label{eq:a} \\
& \hspace{-1cm} b = \hspace{-10mm} & (1-q)(1-z)\sum_{i=0}^{\infty}[z(1-q)]^i f(\ell+i+1+\hat{L}) \label{eq:b} \\
&\hspace{-1cm} d = \hspace{-10mm} & \sum_{i=0}^{\infty}[z(1-q)]^i \left[ q(i+1)f(\ell+i+1+\hat{L}) + \nonumber \right. \\ 
&& \left. (1-q)(1-z)(i+1)f(\ell+i+1+\hat{L}) \right] \label{eq:d}
\end{eqnarray} 
In the special case in which the reward function is constant 
(let this constant be $C$), we obtain:
\begin{eqnarray}
 &a =& C \label{eq:ac} \\
 &b =& C \frac{(1-q)(1-z)}{1-z + qz} \label{eq:bc} \\
 &d =& C \frac{1}{1-z+qz} \label{eq:dc}
\end{eqnarray}
It is of immediate verification that \equaref{ac},\equaref{bc},\equaref{dc}
satisfy \equaref{crucial} for any $q \in [0,1)$ (the case $q=1$ corresponds
to having $a = d$, which is considered separately in Section \ref{sec:single}).

To analyze what happens when $f(\cdot)$ is a decreasing function,
we adopt an iterative approach. We consider a sequence 
of reward functions $\{f_k(\cdot)\}_k$, indexed by $k = 0,1,2,\ldots$,
defined as follows. Let $L = \ell + \hat{L}$ be the minimum path length
experience by random walks traversing the decision point.
We define:
\begin{equation}\label{eq:fk}
f_k(L+i) = \begin{cases} f(L+i) &\mbox{if } 0 \leq i \leq k \\
f(L+k) & \mbox{if } i > k  \end{cases} 
\end{equation}
In words, function $f_k(\cdot)$ matches the actual reward function
$f(\cdot)$ up to hop count $L+k$, while is takes a constant value (equal to
$f(L+k)$ for larger hop count. See Figure \ref{fig:functions}.

\tgifeps{8}{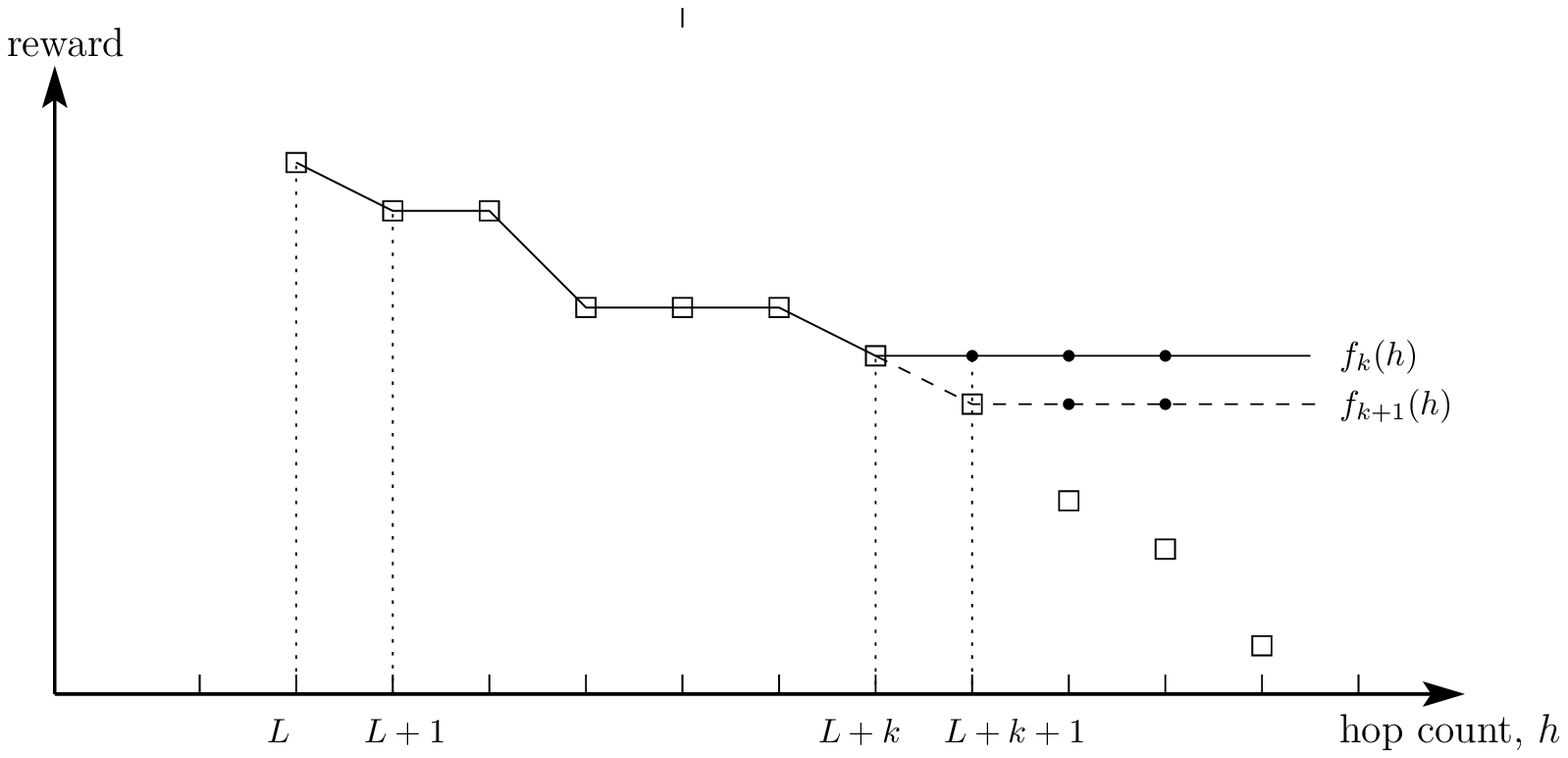}{Example of reward functions $f_k(\cdot)$ and $f_{k+1}(\cdot)$.
Values taken by the actual reward function $f(k)$ are denoted by squares.
Values taken by function $f_k(\cdot)$ (function $f_{k+1}(\cdot)$) are connected by solid (dashed) line.}

In our proof, we will actually generalize the result in Theorem  
\ref{maintheo}, allowing the reward function to be non-increasing for 
values larger than $L$. To simplify the notation, let $f(L) = C$. For $i = 1,2,\ldots$, 
let $f(L+i) = C-\delta_i$, with $\delta_1 > 0$, and $\delta_i \geq \delta_{i-1}$.
   
Let $a_k,b_k,d_k$ ($a_{k+1},b_{k+1},d_{k+1}$) be the entries of matrix 
\equaref{abd} when we assume that rewards are given to edges according
to function $f_k()$ (function $f_{k+1}()$), with $k \geq 0$.
As a first step, we can show that \equaref{crucial} does not hold
already for $k = 0$, i.e., for a reward function which is 
equal to $C$ for hop count $h = L$, and equal to $C-\delta_1$ for
any $h > L$. Indeed, in this case we have:
\begin{eqnarray*}
 &a_0&= C \\
 &b_0&= (1-q)(1-z)\sum_{i=0}^{\infty}[z(1-q)]^i (C-\delta_1) \\
 &   &= \frac{(1-q)(1-z)}{1-z+z q}(C-\delta_1) \\
 &d_0&= (1-z+z q)\sum_{i=0}^{\infty}[z(1-q)]^i (i+1) (C-\delta_1) \\
 &   &= \frac{1}{1-z+z q} (C-\delta_1) 
\end{eqnarray*} 
It can be easily check that $\frac{d_0-a_0}{d_0-a_0+b_0} < z - \epsilon $
for any \mbox{$0 < \epsilon < \frac{\delta_1 (1-z)(1-z + z q)}{(1-q)C - \delta_1}$}.
  
To show that \equaref{crucial} cannot hold for the actual reward function
$f(\cdot)$, it is then sufficient to prove the inductive step
$$ \frac{b_k}{d_k-a_k} \leq \frac{b_{k+1}}{d_{k+1}-a_{k+1}} $$
Note, indeed, that the sequence of functions $\{f_k(\cdot)\}_k$
tends point-wise to $f(\cdot)$.
Now, for any $k$ for which $\delta_{k+1} = \delta_{k}$ there is nothing to prove,
since in this case $\frac{b_k}{d_k-a_k} = \frac{b_{k+1}}{d_{k+1}-a_{k+1}}$.
So let's suppose that $\delta_{k+1} > \delta_{k}$.

We have $a_k = a_{k+1} = C$. We can write $b_k$ as: \vspace{-1mm}
\begin{eqnarray*} 
b_k & = &  \hat{b} + (1-q)(1-z) \sum_{i=k+1}^{\infty}[z(1-q)]^i (C-\delta_{k}) \\
&& =  \hat{b} + (1-q)(1-z)(C-\delta_{k})\frac{[z(1-q)]^{k+1}}{1-z+z q}
\end{eqnarray*}
where \vspace{-2mm}
$$\hat{b} = (1-q)(1-z) \sum_{i=0}^{k}[z(1-q)]^i (C-\delta_{i+1})$$
We can write $b_{k+1}$ as: \vspace{-1mm}
\begin{eqnarray*} 
b_{k+1} & = &  \hat{b} + (1-q)(1-z) \sum_{i=k+1}^{\infty}[z(1-q)]^i (C-\delta_{k+1}) \\
&& =  \hat{b} + (1-q)(1-z)(C-\delta_{k+1})\frac{[z(1-q)]^{k+1}}{1-z+z q}
\end{eqnarray*}
Similarly, we have: \vspace{-1mm}
\begin{eqnarray*} 
d_k & \!\!\!=\!\!\! &  \hat{d} + (C-\delta_{k}) [z(1-q)]^{k+1} \left(k+1+\frac{1}{1-z +z q}\right) \\
d_{k+1} & \!\!\!=\!\!\! &  \hat{d} + (C-\delta_{k+1}) [z (1-q)]^{k+1} \left(k+1+\frac{1}{1-z +z q}\right)
\end{eqnarray*}
where \vspace{-2mm}
$$\hat{d} = \sum_{i=0}^{k}[z (1-q)]^i (i+1)(1-z +z q)(C-\delta_{i+1})$$

We will assume that both $d_k > a_k$ and $d_{k+1} > a_{k+1}$, otherwise the result
is trivial (if $d_{k} < a_k$, then also $d_{k+1} < a_{k+1}$, since $a_{k+1} = a_k$, $d_{k+1} < d_{k}$. If
$d_{k+1} < a_{k+1}$, the normalized ratio of the $\beta$-edge can only tend to zero).
Under this assumption, we can show that 
\begin{equation}\label{eq:fin}
\frac{b_k}{d_k-a_k} < \frac{b_{k+1}}{d_{k+1}-a_{k+1}} 
\end{equation} 
Indeed, plugging in the expressions of $a_k,b_k,d_k,a_{k+1},b_{k+1},d_{k+1}$, 
after some algebra we reduce inequality \equaref{fin} to:
$$ \hat{b}[(k+1)(1-z +z q) + 1]+C(1-q)(1-z) > \hat{d} (1-q)(1-z)$$
At last, recalling the definitions of $\hat{b}$ and $\hat{d}$, we obtain that the
above inequality is satisfied if
\begin{eqnarray*} 
&& \sum_{i=0}^{k}[z(1-q)]^i [(k+1)(1-z+z q) + 1](C-\delta_{i+1}) > \\
&& \sum_{i=0}^{k}[z(1-q)]^i (i+1)(1-z+z q)(C-\delta_{i+1})
\end{eqnarray*} 
which is clearly true, since $k \geq i$ when $i$ varies from 0 to $k$.

We now provide a sketch of the proof for the case in which the random walk
arrives at the decision point having accumulated a random
hop count $\ell_{n+1}$. After long enough time, we can assume that
the probability distribution of $\ell_{n+1}$ has converged to a
random but fixed distribution that no longer depends on $n$.    
Indeed, such distribution depends only on normalized edge weights,
which in the long run converge to constant values.
Let $p_m = \Pp\{\ell_{n+1} = \ell_{\min} + m\}$, $m \geq 0$, where
$\ell_{\min}$ is the minimum hop count that can be accumulated 
at the decision point. We can use $\{p_m\}_m$ to compute expected values of $a$, $b$, $d$
as defined in \equaref{a}, \equaref{b}, \equaref{d}, and apply again 
the Poissonization technique to compute asymptotic values of edge weights. 

Specifically, letting $C = f(l_{\min} + \hat{L})$, we obtain: 
\begin{eqnarray*}
 & \hspace{-5mm} \E[a] = \hspace{-3mm}& \sum_{m=0}^{\infty} p_m (C-\delta_m) \\
 &\hspace{-5mm} \E[b] = \hspace{-3mm}& (1-q)(1-z) \sum_{m=0}^{\infty} \sum_{i=0}^{\infty}[z (1-q)]^i (C-\delta_{m+i+1}) \\
 &\hspace{-5mm} \E[d] = \hspace{-3mm}& (1-z +z q) \sum_{m=0}^{\infty} \sum_{i=0}^{\infty}[z (1-q)]^i (i+1) (C-\delta_{m+i+1}) 
\end{eqnarray*} 
Similarly to before, we prove by contradiction that \equaref{crucial} cannot hold, 
through an iterative approach based on the sequence of reward functions $\{f_k(\cdot)\}_k$.
As basic step of the induction, we take the reward function
$f_0(\cdot)$ equal to $C -\delta_1$ for any hop count larger than $\ell_{\min} + \hat{L}$.
Hence, we have $\delta_{m+i+1} = \delta_1$, $\forall m,i \geq 0$.
It follows that $E[b]$ is exactly the same as in \equaref{b}, and
$E[d]$ is exactly the same as in \equaref{d}. The only quantity
that is different is $\E[a] = p_0 C + (1-p_0)(C-\delta_1) = C-\delta_0$, where
$\delta_0 < \delta_1$ as long as $p_0 > 0$. 
Therefore, whenever there is a non-null probability $p_0$ to reach the decision
point with minimum hop count, the basic induction step proven before
still holds here, by redefining $C$ and $\delta_1$ as $C-\delta_0$ and $\delta_1-\delta_0$, respectively.
One can also prove that the generic iterative step still holds, by following the same lines
as in the basic case. Indeed, one can verify that        
$$ \frac{\E[b_k]}{\E[d_k]-\E[a_k]} \leq \frac{\E[b_{k+1}]}{\E[d_{k+1}]-\E[a_{k+1}]}$$ 
when $\E[d_{k+1}] > \E[a_{k+1}]$.
This concludes the proof of Lemma \ref{lem2}. 
\end{proof} 

\begin{proof}[Proof of Theorem \ref{maintheo} (general case)] 
The proof is exactly the same as in the DAG case, with the difference that
we employ Lemma \ref{lem2} instead of Lemma \ref{lem1} to iteratively prune 
$\beta$-edges from the decision points, leaving only paths from $s$ to $d$ of minimum length.  
\end{proof} 

\vspace{-0mm}
\subsection{The single reward model in general network}\label{sec:single}
We conclude the asymptotic analysis considering the single-reward model
in a general directed network. Given the analysis for the multiple-reward 
model, the single-reward model is almost immediate. Indeed, the expressions 
for $a$ and $b$ (respectively in \equaref{a} and \equaref{b}) are left unmodified, 
as well as their averages $\E[a]$ and $\E[b]$ with respect to hop
count accumulated at the decision point. 
Instead, we have
\begin{equation}\label{eq:eds}
\E[d] = (1-z+z q) \sum_{m=0}^{\infty} \sum_{i=0}^{\infty}[z(1-q)]^i (C-\delta_{m+i+1})
\end{equation}
which is clearly smaller than the $\E[d]$ obtained under the multiple-reward model.
Hence, the basic step of the induction used to prove Lemma \ref{lem2} follows immediately
from the consideration that $\frac{\E[d]-\E[a]}{\E[d]-\E[a]+\E[b]}$ is an increasing function 
of $\E[d]$. 
Moreover, simple algebra shows that the iterative step holds also in the case
of single-reward, allowing us to extend the validity of Lemma \ref{lem2}, and thus
Theorem \ref{maintheo}. 

Last, it is interesting to consider the case of single reward model and 
constant reward function, $f(\cdot) = C$.
We have in this case:
\begin{eqnarray}
 &a =& C \label{eq:acs} \\
 &b =& C \frac{(1-q)(1-z)}{1-z + q z} \label{eq:bcs} \\
 &d =& C \label{eq:dcs}
\end{eqnarray}
Since $a=d$, the matrix exponential takes a different form  
with respect to \ref{eq:asymlem2}, that now reads: 
\begin{equation}\label{eq:1b1}
 \left( \begin{array}{c} 
{\bm E }[\hat{w}(t)] \\
{\bm E }[\dot{w}(t)] \end{array} \right) = e^{a t} 
\left( \begin{array}{cc}
1 & b \\
0 & 1 \end{array} \right)
\left ( \begin{array}{c} \hat{w}(0) \\ \dot{w}(0) \end{array} \right ) 
\end{equation} 
We can show by contradiction that the normalized weight of the $\beta$-edge
cannot tend to any $z > 0$. Indeed, assuming to restart the system
after a long enough number of walks $n^*$
such that $\hat{w}(n^*) \approx \frac{1-z}{z} \dot{w}(n^*)$,
we should have:
$$ e^{a t} \dot{w}(n^*) \sim_e \frac{z}{1-z} \left( e^{a t} \hat{w}(n^*) 
 + e^{a t}\,b\,\dot{w}(n^*) \right)$$
which can only be satisfied if $b = 0$. 
Interestingly, $b$ equals 0 when $q = 1$, i.e., when the $\beta$-edge 
becomes a $\beta^*$-edge. This means that, asymptotically, the probability that 
the random walk makes any loop must vanish to zero.
We conclude that, in the case of a constant single reward model, 
many paths can survive (including non-shortest paths), but 
not those containing loops. In other words, surviving edges must 
belong to a DAG. Simulation results, omitted here due to lack 
of space, confirm this prediction.

a
\section{Transient analysis}\label{sec:transient}
Beyond the asymptotic behavior, it is interesting to consider 
the evolution of edge weights over time. In particular, 
since all non-shortest paths are taken with vanishing 
probability, what law governs the decay rate of such 
probabilities? How does the decay rate depend on system 
parameters, such as network topology and reward function? 
Such questions are directly routed to analogous questions
regarding how normalized edge weights evolve over time,  
as the probability of taking a given path is simply the product of the probabilities 
of taking its edges. Thus, we investigate the transient behavior of 
normalized edge weights. 


\vspace{-0mm}
\subsection{Single decision point}\label{subsec:transingle}
We again start by considering the case of a single 
decision point with two outgoing edges (edge 1 and edge 2), whose initial 
weights are denoted by $w_1[0]$ and $w_2[0]$, respectively. 
Let $\Delta_1 = f(L_1)$ and $\Delta_2 = f(L_2)$ be the rewards 
associated to edge 1 and edge 2, and $L_{1,2}$ the corresponding path lengths.


As discussed in Section \ref{subsec:onepoint}, the dynamics 
of this discrete time system can be usefully embedded 
into continuous time using the Poissonization technique, 
which immediately provides the transient behavior of the 
system in the simple form \equaref{poissonization}. 
To complete the analysis, the solution in continuous time $t$ 
should be transformed back into discrete time $n$. 
Unfortunately, this operation can be done exactly only in the 
trivial case of just one edge. With two (or more) edges,
we can resort to an approximate (yet quite accurate) heuristic 
called depoissonization, which can be applied to all P\'{o}lya urn models governed by invertible 
ball addition matrices~\cite{book}. In this simple topology, assuming $\Delta_1 > \Delta_2$, 
the approximation consists in assuming that ball all extractions that have occurred by time $t$
are associated to the winning edge only (this becomes more and more true with the passing of time), 
which permits deriving the following approximate relation between $n$ and $\bar{t}_n$, where $\bar{t}_n$
is the average time at which the $n$-th ball is drawn:
\begin{equation}\label{eqn}
n \approx \frac{w_1[0]}{\Delta_1}e^{\Delta_1 \bar{t}_n}
\end{equation}
from which one obtains $\bar{t}_n \approx \left( \log \frac{n \Delta_1}{w_1[0]} \right) / \Delta_1$.
Using this approximate value of $\bar{t}_n$ into \equaref{poissonization},
we can approximate the expected values of edge weights after $n$ walks as:
\begin{equation}\label{eq:approxexp2}
\left( \begin{array}{c} 
{\bm E }[w_1[n]] \\
{\bm E }[w_2[n]] \end{array} \right)
\approx \left( \begin{array}{cc}
e^{\Delta_1 \bar{t}_n} & 0 \\
0 & e^{\Delta_2 \bar{t}_n} \end{array} \right) 
\left ( \begin{array}{c} w_1[0] \\ w_2[0] \end{array} \right )
= \left( \begin{array}{c} \Delta_1 n \\
w_2[0] \left( \frac{n \Delta_1}{w_1[0]} \right)^{\frac{\Delta_2}{\Delta_1}} \end{array} \right) 
\end{equation}
The above approximation is not quite accurate for small values
of $n$. In particular, the normalized weight of edge 2, according to \equaref{approxexp2},
can be even larger than the initial value $\frac{w_2[0]}{w_1[0]+w_2[0]}$.
For this reason, for small values of $n$, we improve the approximation 
by assuming that the (average) normalized weight of edge 2 cannot exceed its initial value
at time 0.
Indeed, we can easily find analytically the maximum value of $n$, denoted by $n^*$,
for which we bound the normalized weight of edge 2 to the value $\frac{w_2[0]}{w_1[0]+w_2[0]}$.   
It turns out that $n^* = \frac{w_1[0]}{\Delta_1}$. Note that $n^*$ depends
solely on parameters of the first edge.

Our final approximation for the (average) normalized weight of edge 2 is then:
\begin{equation}\label{eq:finalapprox2}
{\bm E } \left[ \frac{w_2[n]}{w_1[n]+w_2[n]} \right] \approx 
\begin{cases} 
\frac{w_2[0]}{w_1[0]+w_2[0]} & \mbox{if } n \leq n* \\
\frac{1}{ 1+ \frac{w_1[0]}{w_2[0]} \left( \frac{n \Delta_1}{w_1[0]} \right)^{1-\frac{\Delta_2}{\Delta_1}} } & 
\mbox{if } n > n^* 
\end{cases}
\end{equation}
The expression for the (average) normalized weight of edge 1 is then easily derived
as the complement of the above.
%

The value of $n^*$ can be used to separate the transient regime into two 
parts: we call the first one, for $n \leq n^*$, the {\em exploration} phase, 
because during this initial interval there is still no clear winner 
between the competing edges, and random walks explore all possibilities      
with lots of variability in the selected edges.
Instead, we call the second one, for $n > n^*$, the {\em convergence} phase, 
where the winning edge starts to emerge and dominate the competition,
whereas the loosing edge inexorably decays. The behavior of this phase
is much more deterministic than the initial one, especially because
at this point edges have accumulated quite a lot of weight, which individual
random walks cannot significantly modify from one walk to another. These two 
phases and decays are illustrated numerically in Section~\ref{sec:results} (Figure \ref{fig:decayposter}).

Interestingly, from \equaref{finalapprox2} we see that the probability to select edge 2 decays
asymptotically to zero (as $n \rightarrow \infty$) according to the power law
$n^{\frac{\Delta_2}{\Delta_1}-1}$. In particular, the larger the 
ratio between $\Delta_1$ and $\Delta_2$, the faster the decay, which 
cannot however be faster than $n^{-1}$.

\vspace{-0mm}
\subsection{General network: recursive method}\label{subsec:trangen}
We propose two different approaches to extend the transient analysis to a general network.
Our goal is to approximate the evolution of the average weight $\E[w_{i,j}[n]]$ 
of individual edges over time $n$ (where the average is with respect to all 
sample paths of the system). 

The first approach is computationally more expensive but conceptually simple and 
surprisingly accurate in all scenarios that we have tested (see Section \ref{sec:results}). 
It is based on the simple idea of making a step-by-step, recursive approximation 
of $\E[w_{i,j}[n]]$ by just taking the average of \equaref{update}:
\begin{equation}\label{eq:meanupdate}
\E[w_{i,j}[n]] = \E[w_{i,j}[n-1]] + \E[\Delta_{i,j}[n]] 
\end{equation}
where the approximation lies in the computation of $\E[\Delta_{i,j}[n]]$, which is the
expected reward given to edge $(i,j)$ after executing the $n$-th walk.
This quantity can be (approximately) evaluated using just
the set of values $\{\E[w_{i,j}[n-1]]\}_{i,j}$ obtained at step $(n-1)$. 

Indeed, note that $\E[\Delta_{i,j}[n]]$ requires to compute
the distribution of the lengths of paths from $s$ to $d$ containing
edge $(i,j)$. Note that we do not need the complete enumeration of these paths, 
but just the distribution of their length. For this, 
standard techniques of Markov Chain analysis can dramatically reduce the
computational burden, as we will see.   
The fundamental approximation that we make while evaluating this distribution
is the following. First observe that the (averaged) probability to follow a given path at time $n$
is exactly the product of (averaged) independent probabilities to select individual edges.
Unfortunately, the probability to select any given edge corresponds to its
(averaged) {\em normalized} weight at time $n-1$:
$$\E[r_{i,j}[n]] = \E \left[\frac{w_{i,j}[n-1]}{\sum_{k} w_{i,k}[n-1]}\right]$$
which cannot be evaluated exactly, since we do not know the (joint) probability
density function of weights. So we approximate $\E[r_{i,j}[n]]$ by the ratio of averages: 
$$\E[r_{i,j}[n]] \approx \frac{\E[w_{i,j}[n-1]]}{\E[\sum_{k} w_{i,k}[n-1]]}$$ 
which is instead completely known if we have values $\{\E[w_{i,j}[n-1]]\}_{i,j}$.  
In essence, this approximation consists in using the ratio of expectations 
as the expectation of a ratio.     

\tgifeps{3}{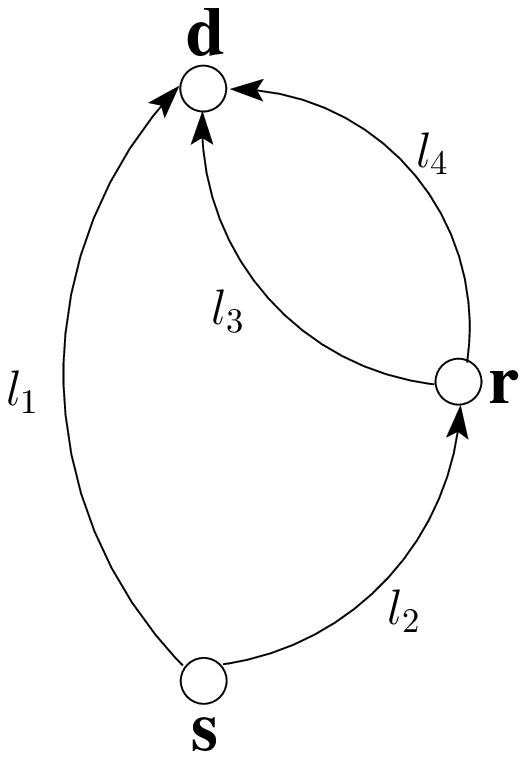}{Example of topology comprising two decision points}

In order to illustrate this recursive approach, consider the topology in Fig.~\ref{fig:twopoint}, 
comprising two decision points: the source node $s$ and the relay node $r$.
The arcs shown in Fig.~\ref{fig:twopoint} do not represent individual edges but 
paths (i.e., sequence of nodes) with lengths denoted by $l_i$, $i=1,\ldots,4$, in hops. 
We also denote by $w_i$, with some abuse of notation, the weight associated to the
first edge of the corresponding path $l_i$. 

The approximate transient analysis of this graph is obtained by the following set of recursive equations:
\begin{equation}\label{eq:weird}
\begin{cases} 
\E[w_1[n]] &= \E[w_1[n-1]] + \E[r_1[n]] \Delta_1 \\
\E[w_2[n]] &= \E[w_2[n-1]] + \E[r_2[n]] \Delta_2[n] \\
\E[w_3[n]] &= \E[w_3[n-1]] + \E[r_2[n]] \, \E[r_3[n]] \Delta_3 \\
\E[w_4[n]] &= \E[w_4[n-1]] + \E[r_2[n]] \, \E[r_4[n]] \Delta_4 
\end{cases} 
\end{equation}
where $\Delta_1 = f(l_1)$, 
$\Delta_2[n] = \E[r_3[n]] f(l_2+l_3) + \E[r_4[n]] f(l_2+l_4)$,
$\Delta_3 = f(l_2+l_3)$, $\Delta_4 = f(l_2+l_4)$.
In the above equations we have denoted the (approximated) normalized weights as $\E[r_i[n]]$.
For example, $\E[r_1[n]] \approx \frac{\E[w_1[n-1]]}{\E[w_1[n-1]]+\E[w_2[n-1]]}$,
and similarly for the other values $\E[r_i[n]]$

The recursive approach can be applied to an arbitrary graph, but in general
it requires to recompute, at each time $n$ (in the case of the single reward model): 
i) the distribution of path lengths from $s$ to $j$ passing through 
edge $(i,j)$, which is an outgoing edge of decision point $i$; ii) the distribution 
of path lengths from $j$ to the destination $d$. The above distributions 
can be computed numerically by solving the transient of discrete-time Markov 
chains with proper absorbing states, but we do not provide the details here. 
Since the overall procedure can be computationally quite expensive in large graphs,
we present in the next section a different, much simpler approach which captures 
the asymptotic law by which average edge weights decay.

\vspace{-0mm}
\subsection{General network: asymptotic power-law decay}\label{subsec:tranlaws}
The asymptotic analysis in Section \ref{sec:asym} shows that the normalized 
weight of all $\beta$-edges (or $\beta^*$-edges) vanishes to zero as $n \to \infty$. 
Can we analytically predict the asymptotic law for such decay? The answer is affirmative, 
and the results offer fundamental insights into how the network 
structure evolves over time.

We start defining a key concept associated to decision points.
\begin{definition}[\bf clock of a decision point] 
The {\em clock} $c_i[n]$ of a decision point $i$ is the expected number 
of random walks that reach $i$ by time $n$:
$$ c_i[n] := \sum_{j = 1}^n \P\{\text{random walk $j$ hits $i$}\} $$ 
\end{definition} 
The clock of a decision point dictates how fast the dynamics of its outgoing edges
evolve with respect to the reference time $n$. As a corollary of 
Theorem \ref{maintheo}, the clock of all decision points traversed by at least one
shortest path is $\Theta(n)$, since any shortest path is asymptotically used with 
non-zero probability. However, decision points not traversed by shortest paths
have clock $o(n)$, and if we put them in sequence we get decision points with increasingly slower clocks.
Nevertheless, we can show that the clock of any decision point is $\omega(1)$.        

Consider, for example, the simple topology in Fig. \ref{fig:twopoint}, and suppose
that $l_1$ is the only shortest path and that $l_3 < l_4$.
Since $l_4$ will be asymptotically used a vanishing fraction of times as compared
to $l_3$ (restricting our attention to the set of random walks passing through $r$, i.e., 
the clock of $r$), we can asymptotically consider decision point $s$
as the sole decision point of the network with two outgoing paths of lengths $l_1$ and $l_2+l_3$.
Hence, we can just apply \equaref{finalapprox2} to compute the power law decay of the $\beta$-edge 
leading to the path with length $l_2+l_3$: 
\begin{equation}\label{eq:powerlaw}
{\bm E } \left[ \frac{w_2[n]}{w_1[n]+w_2[n]} \right] = \Theta(n^{\frac{\Delta_2}{\Delta_1}-1}) 
\end{equation}
where $\Delta_1 = f(l_1)$ and $\Delta_2 = f(l_2+l_3)$.

Moreover, we can again use \equaref{finalapprox2} to compute the scaling order of the (average) clock
of decision point $r$:
\begin{multline}\label{eq:clock}
\E[c_r[n]] = \E\left[\sum_{j = 1}^n \frac{w_2[j]}{w_1[j]+w_2[j]}\right]  = 
\sum_{j = 1}^n \E\left[\frac{w_2[j]}{w_1[j]+w_2[j]}\right] \\ \approx  
\sum_{j = 1}^n \frac{1}{ 1+ \frac{w_1[0]}{w_2[0]} \left( \frac{j \Delta_1}{w_1[0]} \right)^{1-\frac{\Delta_2}{\Delta_1}} }
\!=\! \Theta\! \left(\int_{0}^{n} \!\!\! x^{\frac{\Delta_2}{\Delta_1} - 1} \diff x \!\right) \!\!= \!\Theta(n^{\frac{\Delta_2}{\Delta_1}})
\end{multline}
Note that the clock of $r$ is both $o(n)$ and $\omega(1)$. 
At last, we can compute the power law decay of the first $\beta$-edge of path $l_4$,
by applying again \equaref{finalapprox2} to decision point $r$, with the caveat
of plugging in the clock of $r$ in place of $n$:
$$ {\bm E } \left[ \frac{w_4[n]}{w_3[n]+w_4[n]} \right] = \Theta \left( (n^{\frac{\Delta_2}{\Delta_1}})^{\frac{\Delta_4}{\Delta_3}-1} \right) 
= \Theta \left( n^{\frac{\Delta_2}{\Delta_1} (\frac{\Delta_4}{\Delta_3}-1)} \right) $$
where $\Delta_3 = f(l_2+l_3)$ and $\Delta_4 = f(l_2+l_4)$.

A simple algorithm, that we omit here, can recursively compute the clock of all 
decision points (in scaling order), 
and the power law decay exponent of all decaying edges, starting from the source and moving
towards the destination. We will discuss the implications of our results 
in a significant example presented later in Section \ref{subsec:clocks}.

\section{Validation and insights}\label{sec:results}
We present a selection of interesting scenarios explored numerically 
through simulations to confirm our approximate transient analysis 
and offer insights into the system behavior. 

\subsection{Emergence of shortest paths} 
\begin{figure*}[tbp]
\begin{center}
\includegraphics[width=14cm]{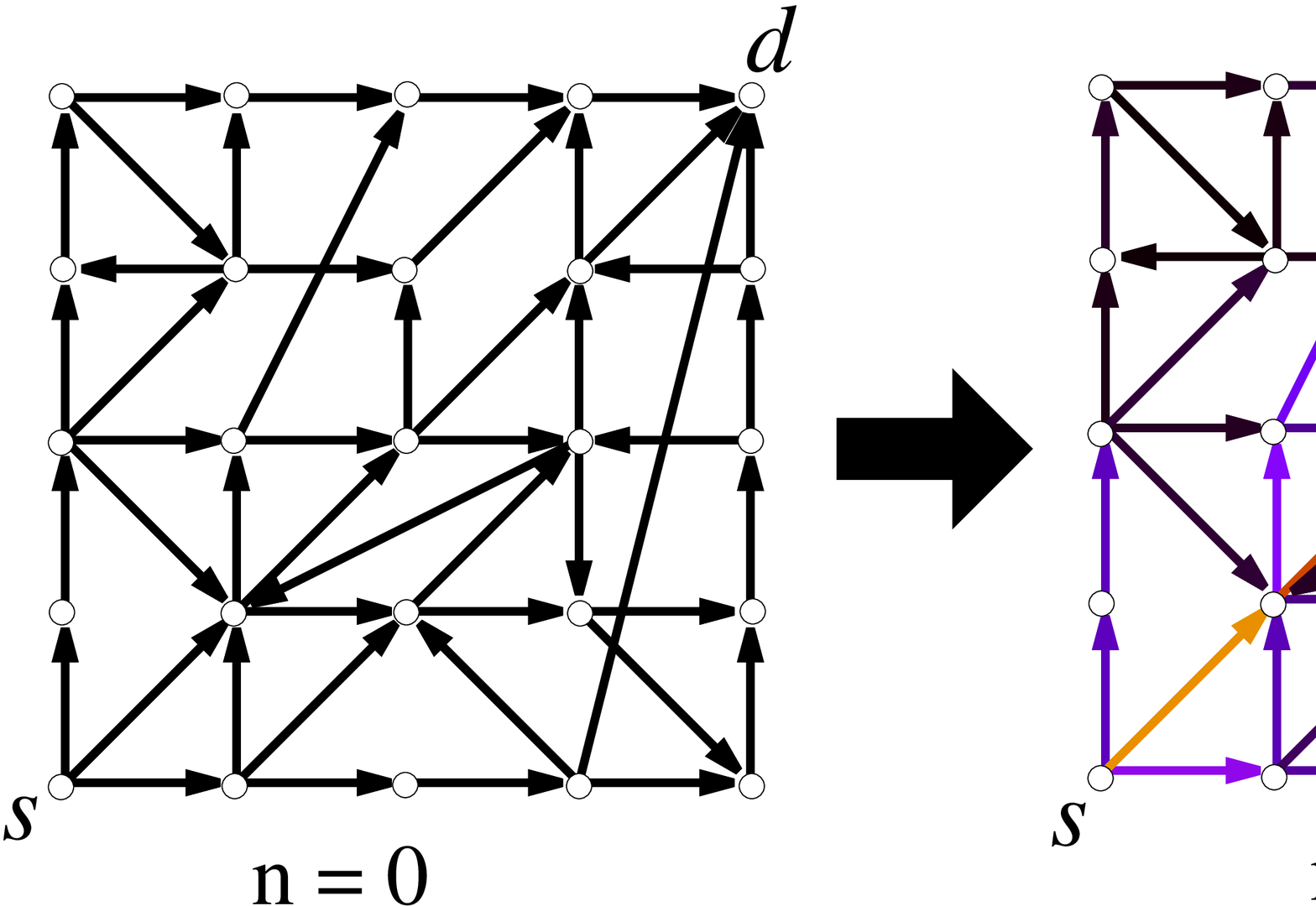}
\caption{\small \label{fig:netsci} From random walks to short walks:  (a) At $n = 0$, all weights are identical; (b) at $n = 10^3$ an edge weight structure starts to emerge along shorter paths; (c) at $n = 10^6$ edge weights along the (two) shortest paths are dominant.}
\end{center}
\vspace{-2mm}
\end{figure*}
We start by considering a 25-nodes network containing a few loops, evolving under the multiple reward model.
Nodes are arranged in a 5x5 grid, with the source located at the bottom left corner and the destination 
at the top right corner. The initial weight on any edge is 1, and the reward function is $f(L) = 1/L$. 
Figure \ref{fig:netsci} (in color, better seen on screen) shows three snapshots
of one system run, at times $n = 0$, $n = 10^3$, $n = 10^6$, where magnitude of edge weights is converted 
into a color code according to a heat-like palette.
We observe that, by time $n = 10^6$, edge weights along the two shortest paths are dominant.
Note that one shortest path (along the diagonal) appears to be stronger than the other (i.e., 
more likely to be used) but this changes from one run to another, since the asymptotic probability 
to use a specific shortest path is a random variable (recall Theorem \ref{maintheo}).    
 
\subsection{Non-monotonous behavior of random walks}
Interestingly, although edge weights increase monotonically,
normalized edge weights (which are the quantities actually steering the random walk 
through the network) can exhibit non-monotonous behavior, even when we
consider their expected values (across system runs).   
We illustrate this on the simple topology of Fig.~\ref{fig:twopoint}, using 
segment lengths $l_1 = 7$, $l_2 = l_3 = 3$, $l_4 = 18$.
Fig. \ref{fig:disturbing} shows the transient of the normalized
weights of the four outgoing edges, comparing simulation results 
(obtained averaging 1,000 runs) and the analytical approximation 
based on the recursive approach \equaref{weird}.
Here initial weights are equal to 1, $f(L) = L^{-2}$. 

\tgifeps{7}{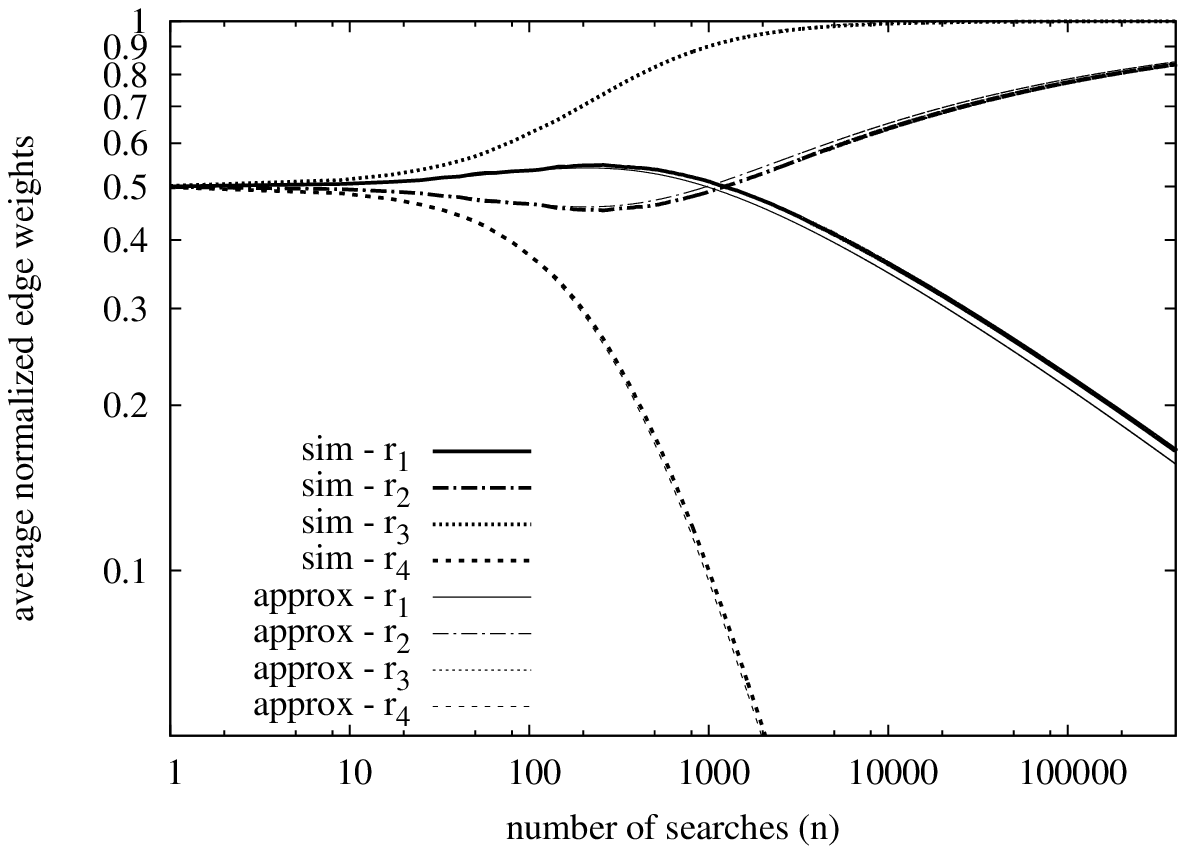}{Transient behavior of normalized weights 
in a simple topology with two decision points. Comparison between
simulation and analytical approximation based on the recursive approach.} 

Besides showing the surprising accuracy of the recursive 
approximation, the plot in Fig. \ref{fig:disturbing} confirms
that normalized edge weights can be non-monotonous (see $r_1$ or $r_2$).
Note that this might appear to contradict a fundamental result that we have
obtained while proving Lemma \ref{lem1} (or Lemma \ref{lem2}), namely, the fact
that the (average) normalized weight of a $\beta$-edge competing against an $\alpha^*$-edge
in non-increasing. However, this result cannot be applied to the first decision 
point (the source $s$) since the assumptions of Lemma \ref{lem1} do not hold here 
(there are no $\alpha^*$-edges going out of $s$). 

\subsection{Trade-off between exploration and convergence}
What happens when we change the reward function $f(L)$? How is the transient
of a network affected by taking a reward function that decreases faster or slower
with the hop count? We investigate this issue in the case of a single decision point,  
considering the family of reward functions $f(L) = L^{-\phi}$
where we vary the exponent $\phi > 0$.

\tgifeps{7}{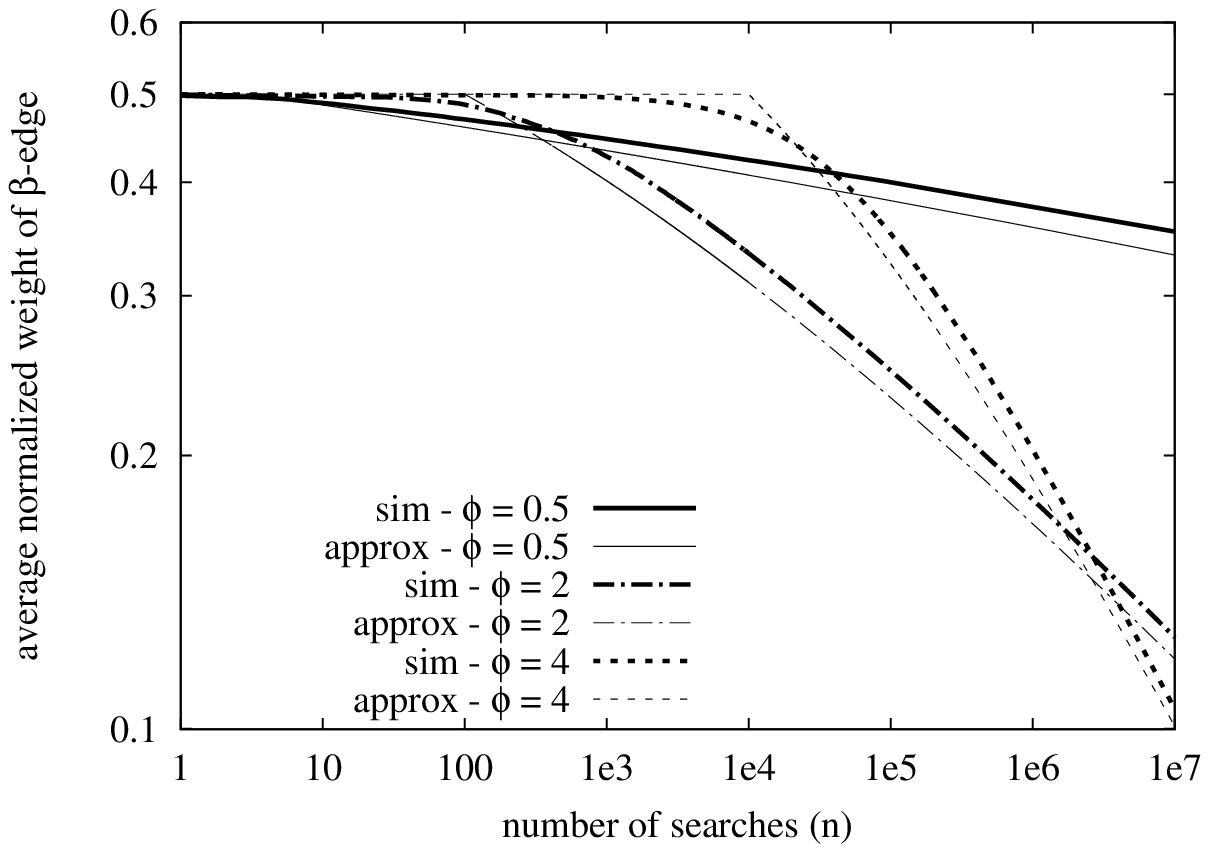}{Transient behavior of a single decision between 
two paths of length 10 and 11, update function $L^{-\phi}$, initial
weights 1.} 

We consider a simple topology in which the source is connected to the destination
by two edge-independent paths of length 10 and 11. The first edges
of these paths have weights $w_1[n]$ and $w_2[n]$, respectively, with
initial value 1.
Figure \ref{fig:decayposter} shows the transient behavior
of the average normalized weight $\E[w_2[n]/(w_1[n]+w_2[n])]$,
for three different values of $\phi = 0.5$, 2, 4, comparing simulation 
results (averaging 1,000 runs) with our analytical
approximation \equaref{finalapprox2}.
 
We observe an interesting trade-off between exploration and convergence. 
Note the role of the threshold $n^*$ introduced in Section \ref{subsec:transingle},
here equal to $n^* = \frac{w_1[0]}{\Delta_1} = 10^\phi$.
Thus, the duration of the exploration phase grows exponentially with $\phi$.  
Moreover, the exponent of the asymptotic power law decay \equaref{powerlaw}
is equal to $\Delta_2/\Delta_1-1 = (10/11)^\phi -1$. Thus, larger $\phi$ leads 
to larger exponent and thus faster convergence. Therefore, as $\phi$ increases 
(corresponding to a reward function that decreases much more rapidly with 
the hop count), convergence is asymptotically faster but exploration requires 
much more time. Intuitively, reward functions that decay too fast with path lengths
require many repetitions of the WRW before some initial structure emerges. However, 
when a structure does emerge, they will quickly drive the WRW deeper and deeper into it.


\subsection{Slowing-down clocks}\label{subsec:clocks}
Consider the network illustrated in Figure \ref{fig:multiple},
comprising of an long sequence of decision points indexed by $1,2,\ldots$.
Each decision point is directly connected to the destination $d$ and to 
the next decision point in the sequence. The source coincides with node 1.

This scenario provides interesting insights into the impact of slowing-down 
clocks on the evolution of the system structure, and will also illustrate
the calculation of the asymptotic power-law decay of $\beta$-edges.

\tgifeps{4}{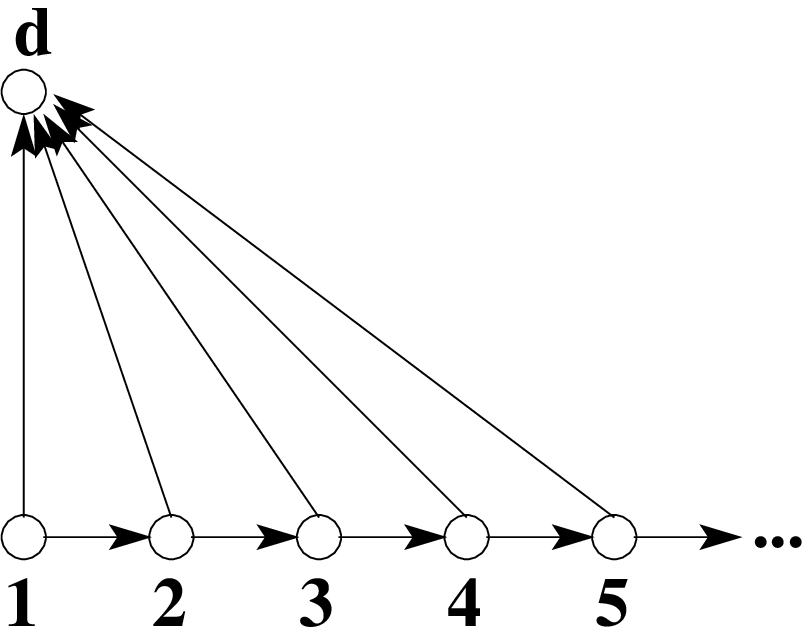}{Network with multiple decision points in sequence.}

Note that, asymptotically, random walks reaching decision point $i$
will end up going directly to $d$ instead of the next decision point. 
This means that, asymptotically, each decision point can be studied in 
isolation, considering two outgoing edges: an $\alpha^*$-edge belonging to 
a path of length $i$, and a $\beta^*$-edge belonging to a path of length $i+1$.

Following this reasoning, we can iteratively compute the power law decay of 
all $\beta^*$-edges, along with the scaling law for the clocks of the 
respective decision points, using the formulas introduced in Section \ref{subsec:tranlaws}.
Let $e_{\beta}^i$ be the scaling exponent of the outgoing $\beta^*$-edge of node $i$, 
and $e_{c}^i$ be the scaling exponent of node $i$'s clock.
Using \equaref{powerlaw} we have $e_{\beta}^1 = \frac{f(2)}{f(1)}-1$,
and from \equaref{clock} we get $e_{c}^2 = \frac{f(2)}{f(1)}$.
Subsequently, we can derive  $e_{\beta}^2 = \frac{f(2)}{f(1)}\left( \frac{f(3)}{f(2)} -1 \right) = \frac{f(3)-f(2)}{f(1)}$
and from this obtain $e_{c}^3 = 1 + \frac{f(3)-f(2)}{f(1)}$.
We then have $e_{\beta}^3 = \frac{(f(1)+f(3)-f(2))(f(4)-f(3))}{f(1)f(3)}$, 
and so on\footnote{We lack a general closed-form expression for $e_{\beta}^i$
or $e_{c}^i$.}.   

\tgifeps{6.5}{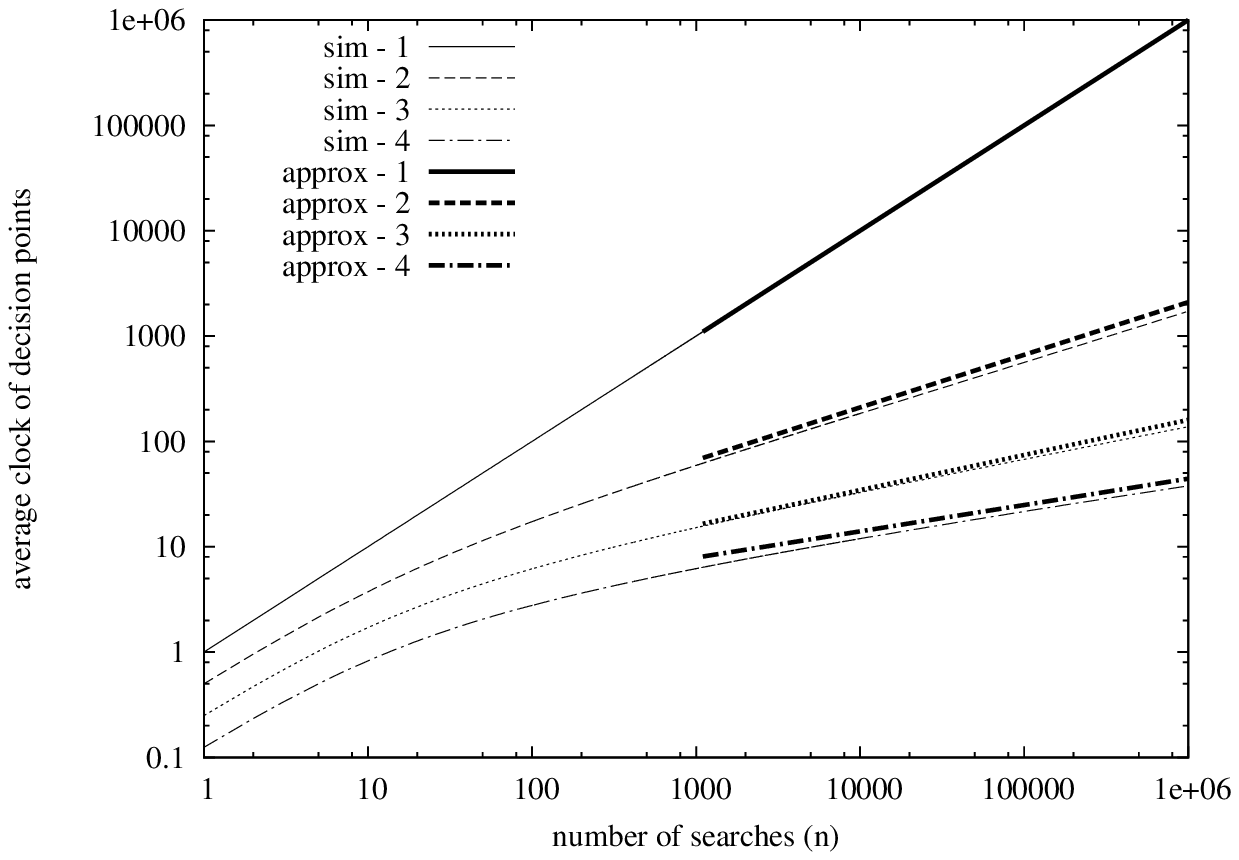}{Clocks associated to the first four decision points
of the network in Fig. \ref{fig:multiple}.}

\tgifeps{6.5}{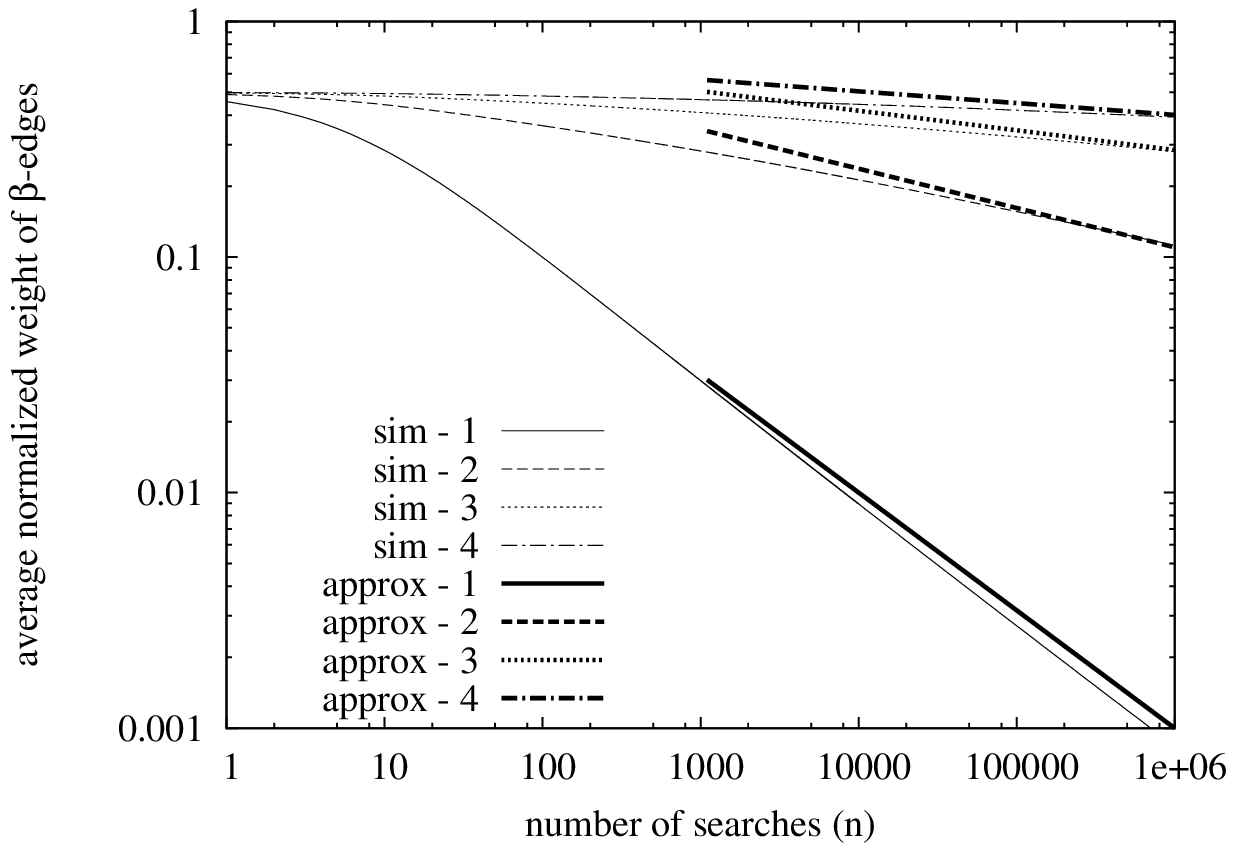}{Normalized weight of $\beta$-edges going out of the first four 
decision points of the network in Fig. \ref{fig:multiple}.}

In Figures \ref{fig:clockexp} and \ref{fig:decayexp} we compare simulation and 
analytical results for the first four decision points shown in Fig. \ref{fig:multiple}, 
considering initial weights equal to 1, $f(L) = L^{-1}$. 
Analytical predictions for the power-law exponents are represented  
by segments placed above the corresponding simulation curve (note the log-log scale).
Besides showing the accuracy of the analytical prediction, results in Fig.
\ref{fig:clockexp} and \ref{fig:decayexp} illustrates an important
fact: the network structure (i.e., weights on edges) is left 
essentially unmodified as we move away from the shortest path. 
Note that this is not quite evident from the math, which predicts that the 
clock of any decision point in the sequence diverges. 
In practice, clocks of decision points sufficiently far from the shortest 
paths evolves so slowly that we can essentially ignore
the perturbations caused by the random walks. 
Hence, sufficiently far regions of a large network preserve their
initial \lq plasticity', allowing them to be used for other purposes. 

\subsection{Transient analysis of the complete graph}\label{subsec:complete} 
As a final interesting case, we consider the complete graph with $m$ nodes, where 
the shortest path has length 1 and every node has cycles of all lengths. 
Without lack of generality, let the source and destination correspond to 
nodes 1 and $m$, respectively.

The asymptotic decay exponent of $\beta$-edges can be easily computed
following the approach in Section \ref{subsec:tranlaws}.
Due to symmetry, there are essentially two types of decision point to analyze:
the source node (node 1), and any other node different from the source 
and the destination, for example, node 2.

Node 1 will have, asymptotically, one surviving edge traversed by the unique shortest path of length 1, 
and $m-2$ decaying $\beta$-edges traversed by paths whose average length will converge to 2.
Hence, the decay exponent of any $\beta$-edge going out of $s$, such as edge $(1,2)$, 
is $\frac{f(2)}{f(1)}-1$. The clock of decision point 2 will then run 
with scaling exponent $\frac{f(2)}{f(1)}$.

Node 2 will have, asymptotically, one surviving edge traversed by paths of length 2, 
and $(m-2)$ decaying edges traversed by paths of average length tending to 3. 
Each of them will decay with power law exponent $\frac{f(3)-f(2)}{f(1)}$.

For the complete graph, we have also run the recursive method introduced
in Section \ref{subsec:trangen}, which required us to numerically solve, at each time step,
the transient behavior of different discrete-time Markov chains with structure
similar to that of the complete graph, modified by the introduction of 
proper absorbing states to obtain the path length distributions needed by 
the recursive formulas.

The recursive approach provides a more detailed prediction of the system behavior (at the cost
of higher computational complexity).
In particular, it allows to distinguish, among the $(m-2)$ $\beta$-edges going out
of decision point 2, the special case of edge $(2,1)$. Intuitively, such edge will accumulate
more weight than the $\beta$-edge connecting node 2 to, say, node 3, 
because node 1 is a different decision point with respect to all other decision points (in particular,
it has the smallest average residual path length to reach the destination).
Considering the completely symmetric structure of the rest of the graph,
it turns out that there are, essentially, 5 types of edges having different 
transient behavior\footnote{to avoid complex notation, we take decision point 2 as 
representative of any decision point different from 1, and decision point 3 as
representative of any decision point other than 1 and 2.}:
1) the $\alpha$-edge $(1,m)$; 2) the $\beta$-edge $(1,2)$; 3) the $\alpha$-edge $(2,m)$; 4) the $\beta$-edge $(2,1)$;
4) the $\beta$-edge $(2,3)$;

The results obtained by the recursive method are compared against simulations in Figure \ref{fig:decay50},
for $m = 50$ nodes, initial weights equal to 1, $f(L) = L^{-1}$, and single reward model.
Besides confirming the surprising accuracy of the recursive approximation,
results in Figure \ref{fig:decay50} suggest that, except for the outgoing edges of 
the source node, all other edges are marginally affected by the reinforcement process. 
Indeed, there are so many edges in this network (i.e., available structure) that 
the \lq perturbation' necessary to discover and consolidate the shortest path between two particular nodes
practically does not significantly affect any edge which is not directly connected to the source.

\tgifeps{7}{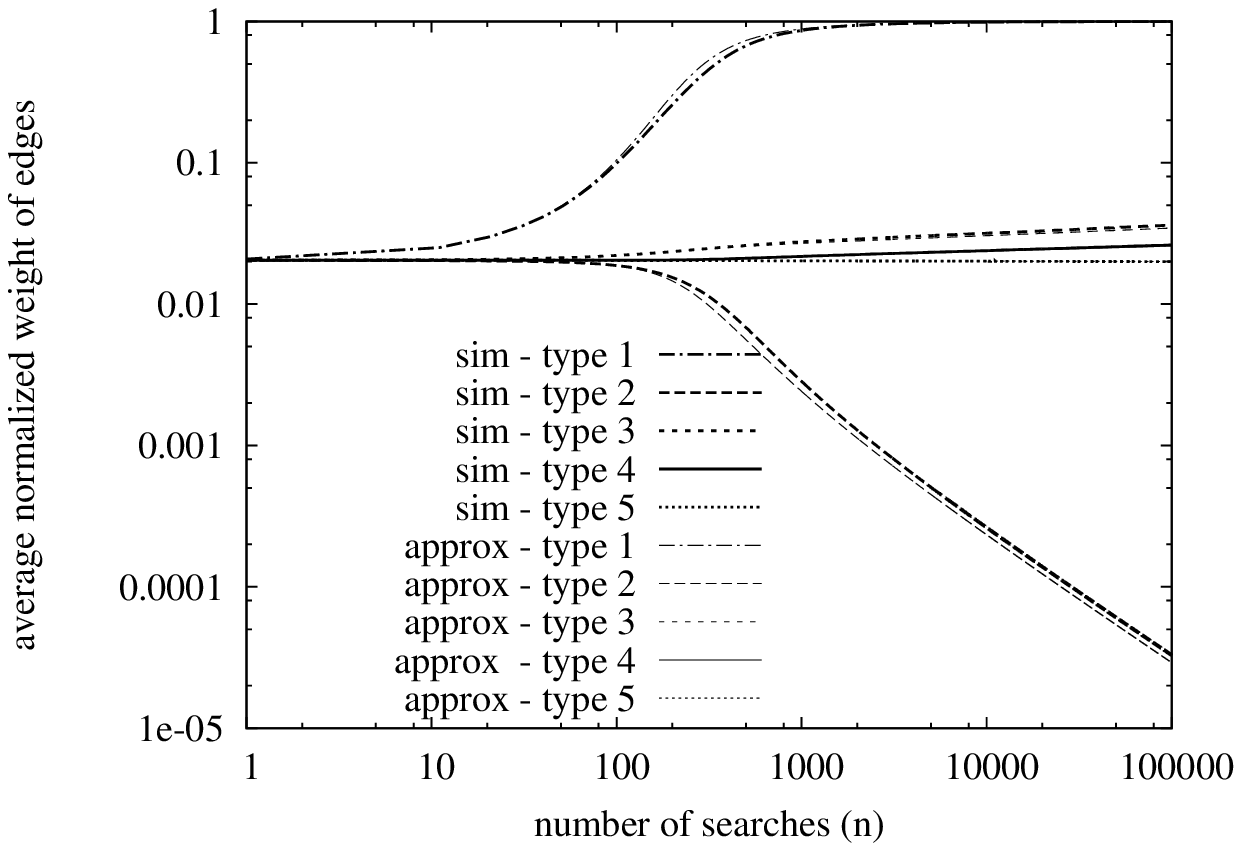}{Transient behavior of the complete graph with 50 nodes, comparing simulation and
recursive approximation, for 5 different edge types.}

\section{Conclusion and future work}\label{sec:conclusion}
Focusing on the important problem of network navigation, 
we have introduced and analyzed a novel, simple model capturing the 
co-evolution of network structure and function performance. 
We have shown how the repetition of a simple network
activity process (WRW with edge reinforcement) is able to build 
over time a network structure that always leads itself to navigate through 
shortest paths, in a surprisingly robust manner. Many variations and 
extensions of the proposed model are possible, which could shed 
light on how information is efficiently found and/or stored in 
biological systems lacking the computational and storage resources 
required to run sophisticated routing algorithms.

\begin{appendices}

\section{Complement to the proof of Lemma \ref{lem2}}\label{app:compl}
The algebraic property related to \equaref{lem2} that we need to prove
holds even if the reward function is constant. Therefore, we will
prove it under the assumption that $\Delta(i) = C$, $\forall i$.
The proof can be extended to a general non-increasing reward function
using the sequence of functions $\{f_k(\cdot)\}_k$ and the iterative approach 
introduced in Section \ref{subsec:gen}, but we omit this extension 
here. 
Moreover, we will first consider the simpler case in which $q = 0$.  
Introducing the following normalized variables: $\alpha = \frac{\dot{w}}{\hat{w}}$ and $\sigma = \frac{C}{\hat{w}}$,
we essentially need to show that 
$$ \frac{1}{\alpha} \sum_{i=0}^{\infty} \left(\frac{\alpha}{\alpha+1}\right)^i
\left[\frac{\alpha + i \sigma}{\alpha+1+(i+1)\sigma}\right] \leq 1$$
for any $\alpha > 0$ and $\sigma > 0$.
Observe that the above expression is exactly equal to 1 for $\sigma = 0$. 
We make another change of variable, introducing
\mbox{$x = \frac{\alpha+1}{\sigma}$}. After some simple algebra,
our target reduces to show that:
$$ g(\alpha,x) = \frac{x+\alpha+1}{\alpha} \sum_{k=x+1}^{\infty} \left(\frac{\alpha}{\alpha+1}\right)^{k-x}
\frac{1}{k} \geq 1$$   
for any $\alpha > 0$ and $x > 0$. 
We know that, $\forall \alpha > 0$, $\lim_{x \to \infty} g(\alpha,x) = 1$. 
Therefore, for arbitrarily small $\epsilon > 0$, there exists an $x_\epsilon$ such that, for $x > x_\epsilon$,
$g(\alpha,x) \geq 1 - \epsilon$.

We can show that, if $g(\alpha,x) \geq 1-\epsilon$, than $g(\alpha,x-1) > 1-\epsilon$, for any $\alpha > 0$
and $\epsilon \geq 0$.
Indeed, we have:
\begin{eqnarray*} 
g(\alpha,x-1) &&= \frac{x+\alpha}{\alpha} \sum_{k=x}^{\infty} \left(\frac{\alpha}{\alpha+1}\right)^{k-x+1} \frac{1}{k} \\ 
&&= \frac{x+\alpha}{\alpha} \frac{\alpha}{\alpha+1}
\sum_{k=x}^{\infty} \left(\frac{\alpha}{\alpha+1}\right)^{k-x}
\frac{1}{k} \\ 
&&= \frac{x+\alpha}{\alpha+1}\left[g(\alpha,x)
\frac{\alpha}{x+\alpha+1}+\frac{1}{x}\right] \\ 
&&\geq (1-\epsilon) \frac{x+\alpha}{x+\alpha+1}\frac{\alpha}{\alpha+1}
+ \frac{1}{x}\frac{x+\alpha}{\alpha+1} > 1-\epsilon
\end{eqnarray*}
as can be easily checked.
Note that, by recursion, if $g(\alpha,x) \geq 1- \epsilon$, than
$g(\alpha,x-m) > 1-\epsilon$ for any $m$ such that $x-m > 0$.
Armed with this result, we can easily prove that $g(\alpha,x) > 1$ for any $x > 0$.
Indeed, suppose, by contradiction, that $g(\alpha,x_\sigma) < 1$ at a given point $x_\sigma$.
Then we can write $g(\alpha,x_\sigma) = 1 - 2 \epsilon$.
Now, we build a sequence of values $\{x_\sigma+1, x_\sigma+2, \ldots, x_\sigma+m, \ldots\}$ which eventually 
enters the stripe $[1-\epsilon,1+\epsilon]$, since $\lim_{x \to \infty} g(\alpha,x) = 1$.
Therefore there exists a sufficiently large $m \geq 1$ such that $g(\alpha,x_\sigma + m) > 1-\epsilon$.
But then it must be $g(\alpha,x_\sigma) > 1-\epsilon$,
which contradicts the hypothesis that $g(\alpha,x_\sigma) = 1 - 2 \epsilon$.
It remains to prove that $g(\alpha,x)$ cannot be identically equal to 1 at all points $x > 0$. 
Again, this can be proven by contradiction: suppose that $g(\alpha,x) = 1$ at any $x$.
Considering a generic point $x_\sigma + 1$, $g(\alpha,x_\sigma+1) = 1$ implies that $g(\alpha,x_\sigma) > 1$,
which contradicts the hypothesis.  
The more general case in which $q > 0$ can be treated essentially in the same way,
but requires more tedious algebra. In this case, we need to show that:
\begin{equation*}
g(\alpha,x,q) = \sum_{k=x+1}^{\infty} \left[\frac{\alpha (1-q)}{\alpha+1}\right]^{k-x} \frac{1}{k}
\left[ \frac{\alpha+1+x}{\alpha (1-q)} + x q \alpha \right] \geq 1 - \alpha q 
\end{equation*}
Evaluating the expression of $g(\alpha,x-1,q)$, one can again show that,
if $g(\alpha,x,q) \geq 1-\epsilon$, then $g(\alpha,x-1,q) > 1-\epsilon$, for any $\alpha > 0$,
$q \geq 0$, $\epsilon \geq 0$, and repeat the arguments
adopted in the case $q = 0$.

\end{appendices}

\bibliographystyle{plain}
\bibliography{main}

\begin{thebibliography}{10}

\bibitem{abbott2012human}
Joshua~T Abbott, Joseph~L Austerweil, and Thomas~L Griffiths.
\newblock Human memory search as a random walk in a semantic network.
\newblock In {\em NIPS}, pages 3050--3058, 2012.

\bibitem{wso}
D.~S. Bernstein and W.~So.
\newblock Some explicit formulas for the matrix exponential.
\newblock {\em IEEE Transactions on Automatic Control}, 38(8):1228--1232, Aug
  1993.

\bibitem{codling2008random}
Edward~A Codling, Michael~J Plank, and Simon Benhamou.
\newblock Random walk models in biology.
\newblock {\em Journal of the Royal Society Interface}, 5(25):813--834, 2008.

\bibitem{davis1990reinforced}
Burgess Davis.
\newblock Reinforced random walk.
\newblock {\em Probability Theory and Related Fields}, 84(2):203--229, 1990.

\bibitem{dorigo2005ant}
Marco Dorigo and Christian Blum.
\newblock Ant colony optimization theory: A survey.
\newblock {\em Theoretical computer science}, 344(2):243--278, 2005.

\bibitem{Dorigo:2004}
Marco Dorigo and Thomas St\"{u}tzle.
\newblock {\em Ant Colony Optimization}.
\newblock Bradford Company, 2004.

\bibitem{durrett}
Rick Durrett.
\newblock {\em Probability: theory and examples}.
\newblock Cambridge University Press, 2010.

\bibitem{huang2006data}
Huilong Huang, John~H Hartman, and Terril~N Hurst.
\newblock Data-centric routing in sensor networks using biased walk.
\newblock In {\em IEEE Conf. on Sensor and Ad Hoc Communications and Networks
  (SECON)}, pages 1--9, 2006.

\bibitem{hbp}
Human brain project, 2013.

\bibitem{stoica2003chord}
{I. Stoica et al.}
\newblock Chord: a scalable peer-to-peer lookup protocol for internet
  applications.
\newblock {\em IEEE/ACM Transactions on Networking}, 11(1):17--32, 2003.

\bibitem{Janson1}
Svante Janson.
\newblock Functional limit theorems for multitype branching processes and
  generalized p\'{o}lya urns.
\newblock {\em Stochastic Processes and their Applications}, 110(2):177 -- 245,
  2004.

\bibitem{Janson2}
Svante Janson.
\newblock Limit theorems for triangular urn schemes.
\newblock {\em Probability Theory and Related Fields}, 134(3):417--452, 2005.

\bibitem{klein}
Jon Kleinberg.
\newblock The small-world phenomenon: An algorithmic perspective.
\newblock In {\em ACM Symposium on Theory of Computing}, STOC'00, pages
  163--170, 2000.

\bibitem{leibnitz2006biologically}
Kenji Leibnitz, Naoki Wakamiya, and Masayuki Murata.
\newblock Biologically inspired self-adaptive multi-path routing in overlay
  networks.
\newblock {\em Communications of the ACM}, 49(3):62--67, 2006.

\bibitem{book}
Hosam~M. Mahmoud.
\newblock {\em P\'{o}lya Urn models}.
\newblock Chapman \& Hall/CRC, 2008.

\bibitem{passarella2012survey}
Andrea Passarella.
\newblock A survey on content-centric technologies for the current internet:
  {CDN} and {P2P} solutions.
\newblock {\em Computer Communications}, 35(1):1--32, 2012.

\bibitem{pemantle2007survey}
Robin Pemantle et~al.
\newblock A survey of random processes with reinforcement.
\newblock {\em Probab. Surv}, 4(0):1--79, 2007.

\bibitem{sadeh2015emergence}
Sadra Sadeh, Claudia Clopath, and Stefan Rotter.
\newblock Emergence of functional specificity in balanced networks with
  synaptic plasticity.
\newblock {\em PLoS Comput Biol}, 11(6):1--27, 06 2015.

\bibitem{saerens2009randomized}
Marco Saerens, Youssef Achbany, Fran{\c{c}}ois Fouss, and Luh Yen.
\newblock Randomized shortest-path problems: Two related models.
\newblock {\em Neural Computation}, 21(8):2363--2404, 2009.

\bibitem{seung2012connectome}
Sebastian Seung.
\newblock {\em Connectome: How the brain's wiring makes us who we are}.
\newblock Houghton Mifflin Harcourt, 2012.

\bibitem{smouse2010stochastic}
Peter~E Smouse, Stefano Focardi, Paul~R Moorcroft, John~G Kie, James~D
  Forester, and Juan~M Morales.
\newblock Stochastic modelling of animal movement.
\newblock {\em Phil. Trans. Royal Soc. of London B: Biological Sciences},
  365(1550):2201--2211, 2010.

\bibitem{sutton1998reinforcement}
Richard~S Sutton and Andrew~G Barto.
\newblock {\em Reinforcement learning: An introduction}.
\newblock MIT press, 1998.

\bibitem{tsitsiklis2003convergence}
John~N Tsitsiklis.
\newblock On the convergence of optimistic policy iteration.
\newblock {\em The Journal of Machine Learning Research}, 3:59--72, 2003.

\end{thebibliography}

\end{document}